\documentclass{article}

\usepackage{iclr2027_conference,times}
\iclrfinalcopy

\usepackage[T1]{fontenc}
\usepackage[utf8]{inputenc}
\usepackage{amsmath,amssymb,amsthm,mathtools}
\usepackage{graphicx,booktabs,tabularx,longtable,array,microtype}
\usepackage{placeins}
\usepackage{xurl,hyperref}

\hypersetup{
  hidelinks,
  pdfauthor={Shunye Wang, Haochen Wen, Shuo Li Liu, Xuanyi Wang, Lihao Liu, Zhongying Deng},
  pdftitle={Transferability of Learned States in Neural PDE Solvers}
}

\newcommand{\E}{\mathbb E}
\newcommand{\norm}[1]{\left\lVert#1\right\rVert}
\newcommand{\state}{\phi}
\newcommand{\control}{\phi_0}
\newcommand{\contract}{\mathcal C}
\newcommand{\deploy}{\mathcal A}

\newtheorem{proposition}{Proposition}
\newtheorem{corollary}[proposition]{Corollary}

\newcolumntype{Y}{>{\raggedright\arraybackslash}X}

\title{Transferability of Learned States in Neural PDE Solvers}

\author{
Shunye Wang$^1$ \quad
Haochen Wen$^2$ \quad
Shuo Li Liu$^3$ \quad
Xuanyi Wang$^4$ \quad
Lihao Liu$^5$ \quad
Zhongying Deng$^2$\\
$^1$ University of Chicago \quad
$^2$ University of Cambridge \quad
$^3$ Princeton University\\
$^4$ ByteDance Inc. \quad
$^5$ Amazon
}

\begin{document}

\maketitle
\lhead{Preprint}

\begin{abstract}
Assessing useful reuse in neural PDE solvers is challenging: final accuracy can reflect source learning and target-time computation. Our reuse contract separates solution accuracy, learning contribution, and numerical utility through paired state comparisons, matched target information and budgets, and cost accounting. A literature audit extracts 18 version-specific protocol records from 12 papers, documenting retained states, target-time resources, and reported controls. For a fixed linear system and residual tolerance, we construct two initial guesses with identical solution-error, energy-error, and residual norms, reaching the same solution with different conjugate-gradient (CG) iteration counts. Across 240 source-training trajectories, two linear PDE families, Fourier neural operators and convolutional networks, a fixed predictor's benefit reverses across correction algorithms. Among pairs with both relative prediction errors $\leq5\%$ on 64 in-distribution tasks ($63\times63$ interior grids), reductions in all three norms accompany more CG iterations, at mean taskwise rates of 23.5\% and 23.9\% in two libraries. Work-based selection saves 2.50--3.33 CG iterations on held-out in-distribution tasks; matched adaptation demonstrates finite-budget pretraining value. Independent batches confirm a 0.73\% complete online saving for one physics-trained Fourier neural operator against zero-initialized Poisson-preconditioned CG. Reuse requires matched state comparisons and downstream computational evidence.
\end{abstract}

\section{Introduction}
\label{sec:intro}
Choosing a reusable neural PDE solver requires more than identifying the most accurate predictor. A deployed model may supply an initial field, undergo target adaptation, or be called repeatedly during correction. The practical decision is whether retaining its learned state improves the intended computation relative to a specified alternative. Initial prediction error and final acceptance describe different aspects of that computation; neither alone establishes its benefit. We ask which judgments about learning remain supported when the corrector, comparison state, target information, and cost boundary are made explicit.

Physics-informed neural networks (PINNs) use equations in learning; neural operators and pretrained solvers retain information across tasks \citep{raissi2019,fno2020,pino2021}. Equations can also enter target-time adaptation and acceptance. Consequently, final accuracy alone cannot separate source learning from work performed on the target problem. Conversely, eventual convergence from an unlearned start does not negate a finite-budget contribution from learning. Our controlled population: supervised and residual-augmented FNO/CNN initializations, label-based target adaptation, and two explicitly scoped native-interface checks. 

This question has direct precedents. Solver-aware warm-start training optimizes downstream losses; hybrid solvers exploit spectral complementarity; numerical audits expose weak baselines and inconsistent costs \citep{sambharya2024,solverloop2020,hints2022,fcg2024,weakbaseline2024}. We study these issues through a controlled comparison that holds trained states fixed while changing the numerical interface, and keeps target selection information and permitted classical reuse explicit. The contribution is the resulting evidence about which judgments reverse, how the mismatch appears among accurate trained states, and which positive findings survive matched controls.

A complementary literature audit makes these distinctions explicit in 18 version-specific protocol records from 12 papers within a broader scoped catalogue, linking reported reuse claims to retained state, target-time resources, and comparison scope (Section~\ref{sec:contract}; Appendix~\ref{app:protocol-records}).

The main evidence connects three questions. \textbf{First, do accurate predictions preserve solver rankings?} A fixed-tolerance CG construction shows that even three initial error summaries together with the accepted output do not determine correction work. In two frozen model libraries, accurate trained states exhibit the corresponding ordering failure: strict improvement in all three summaries still accompanies more CG work, with mean taskwise reversal frequencies of 23.5\% and 23.9\% in the ID/n = 63, 5\%-error stratum. \textbf{Second, which comparison establishes a learning benefit?} Holding the predictor fixed can reverse the sign of its benefit across correctors; changing zero to coarse initialization changes another benefit judgment. Matched adaptation nevertheless retains finite-budget pretraining value. \textbf{Third, does a conditional advantage survive the deployment decision?} Work-based model selection improves some held-out ID comparisons, while independent batches confirm a small 0.73\% online gain under one frozen-state deployment. Other workloads favor classical reuse.

These results connect logical insufficiency, observed trained-state mismatch, and conditional deployment consequences. Executable evidence reconstructs the central comparisons, and exploratory and independent-confirmation cohorts remain separate.
\section{What contribution is being estimated?}
\label{sec:contract}
A task $\tau$ specifies the equation, coefficients, forcing, domain, and conditions. We write
\begin{equation}
 \widehat u_\tau=\deploy_B(\state,I_\tau),\qquad
 \contract=(Q,I,\deploy,B,e,V),
 \label{eq:contract}
\end{equation}
where $Q$ is the target distribution, $I$ the allowed target information, $\deploy$ the complete deployment, $B$ the resource budget, $e$ an independently evaluated error, and $V$ the acceptance rule. We display the compared retained states $(\state,\control)$ separately from the shared contract. Retained state may include weights, normalization, an anchor, a dictionary, an update rule, or executable code. Validation labels and inference-time observations belong to $I$; frozen weights do not imply observation-free prediction.

\paragraph{Correctness is an output claim.}
Does $\widehat u_\tau$ satisfy $V$? Compatible norms give $\norm{u-\widehat u}\leq\norm{L^{-1}}\norm{L\widehat u-f}$ for an invertible linear operator. Sampling, boundary treatment, stability, and discretization govern its interpretation. We distinguish true discrete residual, discrete-reference error, and continuum crosschecks.

\paragraph{Learning contribution is a paired state claim.}
For a specified source-independent control and matched target information,
\begin{equation}
 \Delta_e(B;\state,\control\mid\contract)
 =\E_{\tau\sim Q}\!\left[e(\deploy_B(\control,I_\tau),u_\tau)
                    -e(\deploy_B(\state,I_\tau),u_\tau)\right].
 \label{eq:effect}
\end{equation}
For correction, define $K_\epsilon(\state,\tau)$ as updates to the first verified residual acceptance; the analogous $\Delta_K=\E[K_\epsilon(\control,\tau)-K_\epsilon(\state,\tau)]$ measures update savings, not runtime. Such summaries require the declared common acceptance scope. The accepted-work quantity is undefined for a failed path; we report acceptance first and name any common accepted cohort used for a difference. Capped failures are never silently removed or assigned a successful time. The separate model-selection experiment declares an explicit penalized selection score; it is not a measured time-to-acceptance. Error comparisons at fixed finite budgets do not require both paths to have converged.

The control specifies the question. Zero and coarse numerical initializations test different warm-start contrasts. An identically parameterized scratch network tests source pretraining in adaptation; the frozen source model tests whether adaptation adds value. An arbitrary randomized learned updater need not be a competent replacement for a classical algorithm. A reset must include every source-dependent object, not just selected weights.

\paragraph{Numerical utility is a complete-system claim.}
A useful state contrast need not beat a competent alternative solver. We count the stated online boundary: construction, encoding, transfers, inference, setup, correction, verification, and discarded probes where applicable. Fixed-matrix workloads allow classical factorization and preconditioner reuse. For extra offline cost $C_0$, query count $N$, and matched expected query costs $c_L,c_C$, the cost difference is $C_0+N(c_L-c_C)$. Positive extra offline cost cannot be amortized if $c_L\geq c_C$ under that same workload.

\paragraph{Literature audit: scope and extraction.}
The literature component combines a 104-record, primary-source-checked scoping catalogue with 18 version-specific full-text protocol records from 12 papers. The catalogue search cutoff is 24 September 2026; searches combine method and evaluation terms with backward and targeted forward citation tracing. Each protocol record preserves a source version and locator, retained object, target-time process, reported controls, and evidence scope. Table~\ref{tab:audit-protocol-summary} gives representative distinctions; Appendix~\ref{app:protocol-records} retains all records, and Appendix~\ref{app:extended-literature} details the search boundary. Bibliographic coverage, protocol extraction, and executed experiments are counted separately.

\section{Why initial norms cannot determine correction work}
\label{sec:theory}
\paragraph{Initial losses are also incomplete summaries.}
For a fixed SPD system $Au^\star=b$ and initial error $e_0=u_0-u^\star$, define
\begin{equation}
 S_A(e_0)=\left(\norm{e_0}_2^2,\;e_0^\top Ae_0,\;\norm{Ae_0}_2^2\right).
 \label{eq:summary}
\end{equation}
These are squared solution error, twice the quadratic energy gap, and squared residual: for $\mathcal E(u)=\tfrac12u^\top Au-b^\top u$, $2(\mathcal E(u_0)-\mathcal E(u^\star))=e_0^\top Ae_0$. They connect to supervised, variational, and residual objectives in this fixed discrete setting. If $e_0=\sum_i c_iq_i$ and $Aq_i=\lambda_iq_i$, they retain only three moments of the spectral weights $c_i^2$. Classical exact-arithmetic CG instead satisfies
\begin{equation}
 \norm{e_k}_A^2=\min_{p\in\mathcal P_k,\,p(0)=1}
                 \sum_i\lambda_i c_i^2p(\lambda_i)^2,
 \label{eq:cg}
\end{equation}
where $\mathcal P_k$ contains polynomials of degree at most $k$ \citep{cg1952,saad2003,carson2024}.

\begin{proposition}[Equal summaries and equal accepted outputs, unequal work]
\label{prop:summary}
Fix $A=\operatorname{diag}(1,2,3,4)$, $b=(1,1,1,1)^\top$, exact-arithmetic CG, and acceptance $\norm{b-Au_k}_2/\norm b_2\leq10^{-2}$. Let
\[
 a=(1/\sqrt2,0,1/\sqrt2,0)^\top,\qquad
 v=(1/\sqrt3,1/\sqrt2,0,1/\sqrt6)^\top.
\]
Starting at $u^\star+a$ and $u^\star+v$ gives the same summary $S_A=(1,2,5)$ and the same returned solution $u^\star$, but acceptance takes two and three updates, respectively.
\end{proposition}
\noindent The squared residual sequences are $(5,45/196,0)$ and $(5,2/3,40/867,0)$; $\norm b_2^2=4$. Thus even the pair of initial summaries and accepted output does not determine work for this fixed problem and corrector.

\begin{corollary}[Strictly better in all three summaries need not mean less work]
\label{cor:dominance}
Under the same system, corrector, and tolerance, replacing $v$ by $v/2$ gives $S_A(v/2)=(1/4,1/2,5/4)$, strictly below $S_A(a)$ componentwise, but still requires three rather than two updates.
\end{corollary}
\noindent The second-step relative residual for $v/2$ is $\sqrt{5/1734}\approx0.0537>10^{-2}$. Appendix~\ref{app:identification} supplies exact derivations and a runnable check.

\paragraph{Constructive diagnostic.}
On 32 fresh diffusion tasks, reference-dependent spectral interventions match the three summaries to numerical precision while changing work. At amplitude 1 and tolerance $10^{-8}$, the strongest constructions save about 281–282 CG updates by concentrating mass on at most three spectral atoms. The resulting near-termination is expected. These oracle constructions establish feasibility, not the mechanism of naturally trained predictions or a deployable accelerator. Appendix~\ref{app:spectral} retains the matching defects, rejected construction, and full record counts. PCG requires transformed coordinates whose residual norm need not equal the physical stopping norm.

\section{Controlled study and evidence boundaries}
\label{sec:design}
Table~\ref{tab:protocol} summarizes the primary comparisons and native-interface checks. The core study solves variable-coefficient diffusion and upwind advection--diffusion on a square with homogeneous Dirichlet conditions. FNO and dilated CNN predictors use either supervised data loss or supervised loss plus discrete residual loss. Correction uses CG/Poisson-PCG/AMG-PCG for diffusion and restarted GMRES/Poisson-GMRES/ILU-GMRES for advection--diffusion. Each corrector has matched zero and coarse initializations; sparse direct solution is a further complete-system control.

An initial wave trains 48 source models; an expanded wave trains 192 trajectories to 6000 updates, covering 96 residual-weight configurations, 48 wider networks, and 48 models on new, larger source datasets. Checkpoints at 1200/3000/6000 share a training schedule. Source validation selects checkpoints and positive residual weights, separately from the data-only arm. The new-data block changes both realization and label count; a later nested-prefix study isolates source count within one pool. Extension training is counted separately.

The core waves contain 494,592 correction endpoints over main suites of 256 and 128 physical tasks, not 494,592 independent experiments. Optimizer seeds are averaged within task; grids, tolerances, repetitions, and budget checkpoints do not add independent tasks. Expanded-wave strata have 16/8/4 tasks at $n=31/63/127$; cross-grid comparisons use common subsets. Reported expanded-wave 95\% intervals are descriptive 4000-resample paired-task bootstraps, conditional on fitted states and unadjusted for multiple comparisons. Independent confirmation cohorts are not pooled with discovery cohorts.

All correctors check the true residual of the same float64 discrete system. Of 460,224 expanded-wave endpoints across three tolerances, 2,591 are capped failures, all at $n=127$ without preconditioning. They are retained in their denominators. Accepted $10^{-8}$ endpoints have maximum relative discrete error $1.04\times10^{-8}$. Appendices preserve every experiment family, interface, acceptance boundary, and sampling unit; the accompanying workbook preserves full aggregate configurations rather than only favorable cells.

\begin{table}[t]
\centering\small
\caption{Primary comparisons and their inferential scope. Counts are physical tasks or batch groups.}
\label{tab:protocol}
\begin{tabularx}{\linewidth}{p{.17\linewidth}p{.20\linewidth}p{.27\linewidth}Y}
\toprule
Comparison & Test units and retained states & Target information and outcome & Evidence status\\
\midrule
Corrector reversal & 16 joint tasks; 2 source datasets, 3 seeds each & Source-only selection; same $10^{-8}$ true residual; update counts & Exploratory paired comparison; second source supports direction\\
Fresh-task selection & 64 test tasks/stratum; 2 libraries of 12 models & Same 32 validation labels/stratum; choices locked before test generation & Revision-stage frozen protocol; no new training or universal time claim\\
Adaptation & 128 tasks/equation; pretrained, scratch, frozen & Matched training/validation labels; fixed-budget prediction error & Controlled comparisons; fixed-80-label selection study is separate\\
Online batch gain & 8 new groups of 32; one fixed model state and node & No retraining; full warm online time; $10^{-8}$ acceptance & Fresh-task confirmation after two discovery groups; groups not pooled\\
Official HINTS & 750 official cases; timing on a 32-case subset & Fixed official weights; one-time versus repeated neural correction & Native 29-interior-DOF check; no added source training\\
Scaled official NOWS & 48 tasks; three trained states & Official VINO objective; 100 epochs; matched zero/coarse controls & Reduced-budget native-code audit; 9,216 timed records\\
\bottomrule
\end{tabularx}
\end{table}

\section{Paired interventions locate the contribution}
\label{sec:results}
\subsection{Keeping the predictor fixed can reverse the benefit judgment}
\begin{figure}[t]
\centering
\includegraphics[width=\linewidth]{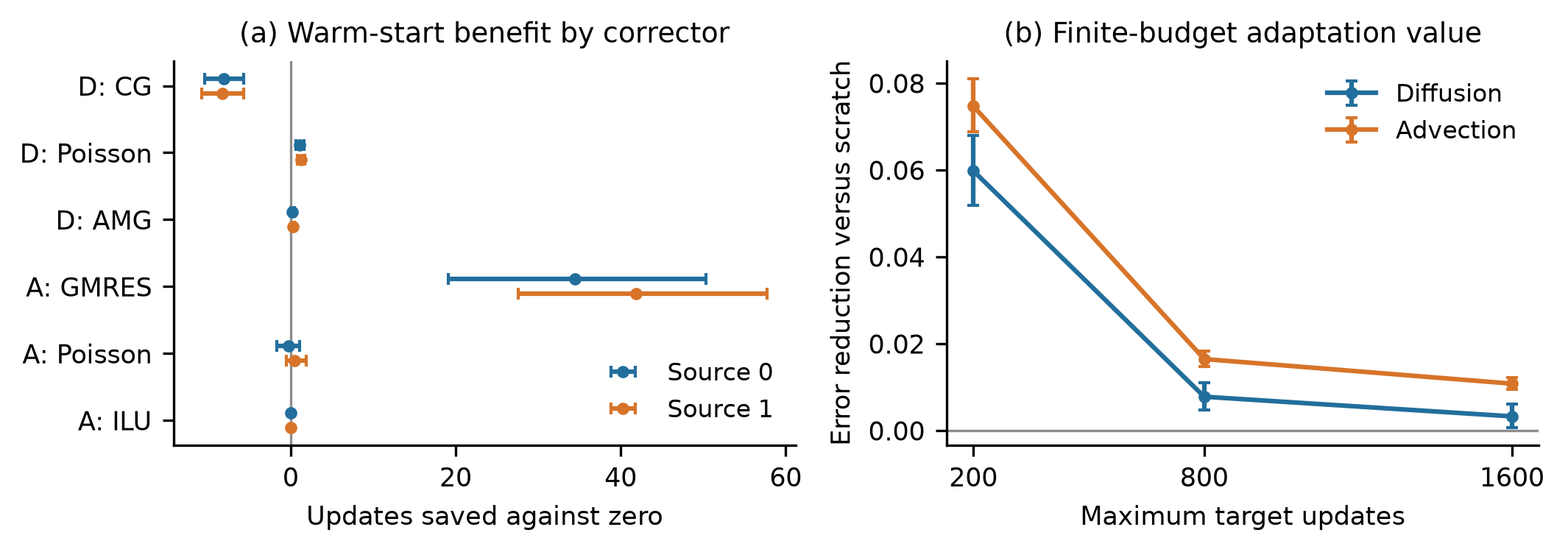}
\caption{\textbf{Contribution is conditional on how state is used.} Left: 6000-update data-FNO update savings versus zero under joint shift, $n=31$, tolerance $10^{-8}$. Right: pretrained-minus-scratch advantage expressed as error reduction, source 0, 256 training plus 64 validation labels, 128 test tasks per equation. Budget points share trajectories. Intervals are descriptive paired-task 95\% intervals.}
\label{fig:core}
\end{figure}
\begin{table}[t]
\caption{One retained predictor, different correction effects. Updates saved relative to zero; positive favors the data-trained FNO. Joint shift, $n=31$, tolerance $10^{-8}$, 6000-update budget. Each cell averages three optimizer seeds within each of 16 physical tasks; brackets are paired-task 95\% bootstrap intervals.}
\label{tab:core}
\centering\small
\begin{tabular}{llrr}\toprule
Equation & Corrector & Source set 0 & Source set 1\\\midrule
Diffusion & CG & -8.06 [-10.44, -5.71] & -8.27 [-10.77, -5.71]\\
Diffusion & Poisson-PCG & +1.125 [0.67, 1.58] & +1.29 [0.83, 1.73]\\
Diffusion & AMG-PCG & +0.23 [0.04, 0.44] & +0.29 [0.10, 0.48]\\
Advection--diffusion & GMRES & +34.42 [19.08, 50.38] & +41.88 [27.56, 57.77]\\
Advection--diffusion & Poisson-GMRES & -0.29 [-1.67, 1.06] & +0.46 [-0.56, 1.85]\\
Advection--diffusion & ILU-GMRES & +0.00 [0.00, 0.00] & +0.00 [0.00, 0.00]\\
\bottomrule\end{tabular}
\end{table}

Table~\ref{tab:core} fixes the learned state within each equation/source and varies the corrector. The diffusion data-FNO joint error improves from 0.1837 to 0.1739 over the common-schedule 1200--6000 budgets, yet at 6000 updates CG still takes 8.06 more updates than zero; Poisson-PCG saves 1.125. Both source datasets reproduce this sign change. For source 0, zero versus learned means are 279.81 versus 287.88 (CG), 58.375 versus 57.25 (Poisson-PCG), and 8.125 versus 7.896 (AMG-PCG): respectively 2.88\% more, 1.93\% fewer, and 2.82\% fewer updates, not runtime changes. The initial relative solution/energy/residual norms are 0.174/0.744/7.390, versus 1/1/1 for zero; source 1 gives 0.162/0.734/7.398. Thus this predictor improves solution and energy error but worsens residual. It is a corrector-reversal example. Section 5.2 tests the stronger ordering question among accurate trained states. Advection--diffusion GMRES saves 34.42/41.88 updates versus zero but needs 25.90/18.44 more than coarse initialization. Strong ILU leaves no observed update headroom on this small grid. A benefit relative to one control does not establish a benefit relative to another.

\paragraph{Strengthening learning helps without aligning every comparison.}
Source-validation-selected residual weights reduce joint diffusion CNN error from 0.1231/0.1438 to 0.0969/0.0979 and shrink CG disadvantages from 4.71/6.35 to 0.13/1.08 updates; the latter intervals include zero. Wider networks and new datasets preserve the main conditional pattern, although capacity does not improve every architecture. Within one fixed source pool, increasing the nested training prefix from 128 to 1024 lowers joint data-FNO errors from 0.1849 to 0.1262 for diffusion and 0.1868 to 0.1316 for advection--diffusion. Nevertheless, diffusion CG still needs 5.81 more updates than zero, and advection--diffusion GMRES saves 45.38 against zero but needs 11.94 more than coarse. All 3,840 nested-source endpoints accept.

\paragraph{The numerical interface and native objective matter.}
Fixed padding, proportional padding, and fixed physical FFT period define different refined-grid deployments; predeclared common-task comparisons avoid choosing a favorable interface after testing. Separately, a scaled audit retains official NOWS Darcy code and its VINO energy objective \citep{nows2026}: three seeds, one 256-example source pool, 100 epochs, 48 test tasks, and 9,216 accepted timed records. At grid 128 and tolerance $10^{-8}$, prediction errors are 2.26\% in distribution and 57.57\% under contrast shift; initialization respectively saves 20.15 CG updates or adds 38.71 versus zero. The shifted case is a loss of predictive quality. This is a native-objective sensitivity check. Appendix~\ref{app:nows-scaled} retains every stratum, grid, and host-specific timing comparison.

\subsection{Accurate trained states can misrank correction work}

The frozen-library study uses two 12-model diffusion libraries and the same 32 new labeled validation tasks per distribution. Choices minimizing prediction error or each corrector's accepted update count, with predetermined tie-breaking rules and failure penalties, were locked before generating 64 held-out tasks per distribution.

\textbf{Ranking mismatch persists among accurate predictions.} On the 64 ID tasks at $n=63$, restricting model pairs to those for which both relative prediction errors are at most $5\%$, the mean taskwise error--CG-work reversal frequencies are $45.47\%$ and $45.57\%$ in the two source libraries. Even when all three initial summaries---solution error, energy error, and initial residual---are strictly smaller for the same prediction, the corresponding conditional reversal frequencies are $23.50\%$ and $23.87\%$, with task-bootstrap $95\%$ intervals $[17.61\%, 30.39\%]$ and $[18.58\%, 29.33\%]$, respectively (64 eligible tasks in each library; Appendix~\ref{app:selection-study}). These are means of within-task fractions. The same dominance subsets have no observed reversal under Poisson-PCG, but many work ties. Thus joint improvement in the three summaries need not preserve CG-work rankings even among accurate trained states. These observations bridge the analytic counterexample to trained predictions.

\textbf{Selection has a separate, conditional consequence.} At $n=63$ in distribution, work-based selection saves 2.50 and 3.33 CG updates in the two source libraries (Table~\ref{tab:selection-new-tasks}), while Poisson-PCG savings are 0.33 and 0.38 updates. Five of 16 unadjusted paired intervals are strictly positive; none is so under joint shift, and five comparisons select identical models. The observed value of work-based selection therefore depends on the correction contract. This is offline model selection, distinct from the paid per-query probes in Section~\ref{sec:deployment}. Computing work scores adds 0.19--4.97 CPU seconds per library/configuration; this selects iterations, not minimum time. For example, one ID/$n=31$ choice saves 0.125 PCG updates while its observed complete CPU path is $2.18\,\mathrm{ms}$ slower.

The 64 validation and 128 test physical tasks are shared across grids and source libraries; all 19,968 validation/test endpoints accept. Intervals condition on the fixed library and one validation draw. Appendix~\ref{app:selection-study} gives the frozen rules, absolute work and costs, and quality-stratum denominators.

\begin{table}[t]
\centering\small
\caption{New-task conditional model selection. Test updates saved by selecting
on validation work rather than prediction error (prediction minus work;
positive is better). Paired-task 95\% intervals use 64 test tasks per distribution. Source libraries 0/1 each contain 12
frozen candidates. No multiplicity adjustment.}
\label{tab:selection-new-tasks}
\begin{tabular}{llcrr}
\toprule
Tasks & Grid & Source & CG & Poisson-PCG \\
\midrule
ID & 31 & 0 & 0.000 [0.000, 0.000] & 0.125 [0.047, 0.219] \\
ID & 31 & 1 & -0.516 [-1.438, 0.391] & 0.016 [-0.078, 0.109] \\
ID & 63 & 0 & 2.500 [1.047, 3.844] & 0.328 [0.219, 0.453] \\
ID & 63 & 1 & 3.328 [1.531, 5.016] & 0.375 [0.266, 0.500] \\
Joint & 31 & 0 & 0.000 [0.000, 0.000] & 0.000 [0.000, 0.000] \\
Joint & 31 & 1 & -0.156 [-1.266, 0.891] & -0.094 [-0.234, 0.031] \\
Joint & 63 & 0 & 0.000 [0.000, 0.000] & 0.000 [0.000, 0.000] \\
Joint & 63 & 1 & 0.547 [-2.406, 3.454] & 0.109 [-0.016, 0.234] \\
\bottomrule
\end{tabular}
\end{table}

\subsection{Matched adaptation reveals value that convergence need not preserve}
With symmetric learning-rate and checkpoint selection, 16 target training plus 64 validation labels give diffusion errors 0.1751 after adaptation, 0.2751 from scratch, and 0.1834 when frozen for source 0. At 256 training labels and larger budgets, scratch approaches the pretrained model (Figure~\ref{fig:core}); the source-1 difference at 1600 updates is 0.00213 with an interval spanning zero. This supports finite-budget adaptation value, not an irreducible asymptotic advantage.

A separate fixed-80-label study controls checkpoint choice on the same trajectory. At source 0, 16/64 training/validation, learning rate $10^{-3}$, and 1600 updates, selecting validation-best rather than final gives diffusion error 0.1711 instead of 0.1970; the paired reduction is 0.02591 [0.02188,0.02989]. Advection--diffusion improves from 0.1527 to 0.1328. With symmetric learning-rate selection, the three tested allocations favor 64/16 over 16/64, but some checkpoint contrasts are unfavorable. All 108 trajectories actually run to 1600 updates: selecting an earlier checkpoint does not establish realized training-cost savings. (Appendix~\ref{app:information}).

\paragraph{A boundary check that revises an earlier negative result.}
In the monotone one-dimensional problem $-\nu u''+\gamma u^3=f$, stronger training changes joint-shift Newton savings from near zero to positive; 56 new tasks confirm $1.000$ [0.708,1.292] saved updates at $n=63$, $10^{-8}$. The earlier extreme-stratum negative sign is not confirmed: $-0.667$ [$-1.668$,0.418]. Coarse initialization remains competitive. The family has a unique discrete root; standard strong monotonicity bounds error to that root, not Newton work or continuum error. (Appendix~\ref{app:nonlinear-long}).

\section{A contribution must survive its full deployment cost}
\label{sec:deployment}
\paragraph{A real, small online gain.}
Following two exploratory batches, eight new 32-query groups confirm a conditional benefit for physics-FNO on diffusion at $n=127$: complete online cost is 27.267 versus 27.469\,ms/query for zero--Poisson-PCG. The paired difference (learned minus zero) is $-0.201$\,ms, with group-bootstrap 95\% interval [$-0.230$,$-0.171$]; all eight groups are faster, a 0.73\% reduction in the ratio of mean costs. All 5,376 repeated query endpoints pass $10^{-8}$ acceptance. Data-FNO is instead 0.278\,ms slower than zero. The new groups exclude the discovery groups, fix one source/optimization seed and node, and include assembly, transfer, inference, setup, correction, synchronization, and residual acceptance, but exclude training and model loading. For an existing model this is an online gain; from-scratch lifecycle benefit additionally depends on offline cost. At the observed mean difference, each extra second of offline cost requires approximately 4965 queries to amortize, assuming unchanged query costs. (Appendix~\ref{app:batch-confirmation}).

\paragraph{The alternative and reuse permissions can change the conclusion.}
In the serial $n=127$ deployment, diffusion costs are 28.57/29.89/40.87\,ms for zero--Poisson, source-0 physics-FNO, and direct solution: beating direct does not identify a neural-initialization benefit. Advection--diffusion instead favors direct (41.14\,ms versus 100.01 for zero--Poisson). On a \emph{different, fixed-matrix workload}, allowing setup once across eight measured right-hand sides gives direct costs 6.46/6.54\,ms for diffusion/advection--diffusion, versus 29.23/121.14 for physics-FNO. Other query counts are projections. Classical retained factorizations must receive the reuse permissions that the workload allows; these observations are not pooled with variable-matrix batches.

\paragraph{Less correction need not mean less total work, for either side.}
In the NOWS audit, coarse initialization uses fewer updates at grid 128 and $10^{-8}$, yet its additional construction and solve make it roughly 0.300--0.330\,s more expensive than the learned path in the reported host/stratum/corrector groups. No such group establishes a paired-bootstrap time advantage of the learned path over matched zero or direct solution. Official matrix assembly (approximately 1.27--1.30\,s) dominates that retained implementation. Likewise, the inexpensive one-dimensional nonlinear banded solve can outweigh an initialization's saved Newton updates. Update counts alone cannot establish utility for a classical initializer either.

\paragraph{Selection, startup, and failure are part of the question.}
On 32 tasks per equation, minimum-residual selection always chooses zero but adds 4.03/4.59\,ms; five-update probes add 5.08/5.70\,ms after charging all discarded candidates and restarting from the selected field. Fresh-process runs accept all 240 solves but take 2.31--2.48\,s for neural paths versus 0.41--0.57\,s for classical paths from process launch to acceptance notification. Finally, all 142 failures among 2,688 boundary-intervention endpoints occur at coefficient-log scale three under the fixed capped iterative paths; direct solution accepts every case. Failed groups receive no successful-subset speed claim. Complete contracts and all configurations are retained in the deployment appendices.

\paragraph{Native repeated correction is a distinct positive control.}
With official HINTS weights on 750 official Poisson cases, repeated correction accepts all cases within 400 updates, while Jacobi and one-time neural initialization accept none. At a 10,000-update cap all accept, with means 138.84, 4992.71, and 3899.80; neural and numerical updates have unequal costs. On 32 timed cases, HINTS takes 1.010\,ms versus Jacobi's 18.809\,ms and a tridiagonal direct solve's 0.0299\,ms. They preserve a clear benefit over the native baseline while separating it from superiority to all alternatives (Appendix~\ref{app:hints-official}).

\section{Related work and the boundary of the claim}
\label{sec:related}
\paragraph{Physics-informed learning and reusable state.}
Neural PDE representations predate modern PINNs \citep{lagaris1998,raissi2019}; residual and variational formulations include DGM and Deep Ritz \citep{dgm2017,deepritz2017}. Analysis separates approximation, optimization, sampling, and stability rather than guaranteeing accuracy from a small training loss \citep{krish2021,ntk2020,precondition2023,deryckreview2024}. DeepONet, FNO, and PINO retain maps across tasks \citep{deeponet2021,fno2020,pino2021}; meta-learning and learned optimizers explicitly permit target computation \citep{metapde2022,metaloss2021,learnedsolver2024}.

\paragraph{Direct precedents.}
Learned fixed-point warm starts already show that equal-distance initializations can converge differently, optimize downstream finite-step losses, and study generalization \citep{sambharya2024}; Solver-in-the-Loop includes learned Poisson-CG initialization \citep{solverloop2020}. Neural super-fidelity and NOWS study warm starts, HINTS exploits spectral complementarity, and FCG-NO learns repeated correction inside flexible CG \citep{superfidelity2025,nows2026,hints2022,fcg2024}. Spectrally safe Newton starts address state-dependent Jacobian indefiniteness \citep{spectralsafe2026}; error-conditioned solvers discuss residual--reconstruction mismatch and learn recurrent updates \citep{ens2026}. Our fixed-SPD intervention instead changes initial-error weights without changing the matrix. A recent study uses imperfect operator priors for nonlinear PINN solution-basin selection \citep{ancestor2026}.

\paragraph{Evidence and numerical utility.}
Strong-baseline and accuracy--cost critiques already exist \citep{grossmann2024,weakbaseline2024}; evaluation of whether outputs support scientific claims is another explicit evidence question \citep{wangevidence2026}. The protocol audit distinguishes frozen parameters from observation-free prediction, a reused solver program from a pretrained solution network, and repeated neural correction from one-time initialization (Table~\ref{tab:audit-protocol-summary}). These distinctions specify which state, information, and computation enter the comparison. The controlled experiments test selected comparisons rather than reproduce every inspected protocol. Classical reduced bases and Krylov recycling also reuse state \citep{hesthaven2016,recycling2006}. Appendix~\ref{app:extended-literature} retains the broader literature map on discretization, geometry, foundation models, generative priors, and executable programs.
\section{Discussion and Limitations}
\label{sec:discussion}
A predictive claim needs independently evaluated error; source-learning value needs a valid state replacement with matched information and budget; numerical utility additionally needs complete costs and competent alternatives. Accepted outputs alone cannot replace these comparisons.

The evidence covers regular geometries, two linear PDE families, and a monotone unique-root nonlinear family, with limited source datasets and small fine-grid cohorts. The HINTS check is small; NOWS uses reduced training. Batch confirmation fixes one state and node. Intervals condition on fitted libraries, not the population of possible training sets; offline costs remain incompletely attributed. Geometry transfer, nonnormal dynamics, branch selection, long-time rollouts, and calibrated inverse posteriors require separate evidence.

Work-based validation improves some ID comparisons, but joint-shift gains and complete-time benefits of that selection policy remain unestablished. Positive finite-budget adaptation and the separate batch-cost confirmation nevertheless show why this audit is not a general rejection of learning. An accurate predictor, a beneficial retained state, and a competitive solver are different claims; each requires evidence under its intended comparison and deployment.

\section{Conclusion}

Learning can change the work required to reach a solution without changing the solution itself. Conversely, a more accurate prediction can require more correction. We make these distinctions testable through a reuse contract separating output correctness, conditional learning contribution, and complete numerical utility. For fixed-tolerance CG, identical solution-error, energy-error, and residual summaries—even together with the same accepted solution—do not determine correction work. The corresponding ordering failure also occurs among trained predictions with at most 5\% relative error, including pairs that improve all three summaries simultaneously. Holding the learned state fixed can reverse its benefit across correctors; changing the numerical control can reverse the comparison again. The computational value of reuse therefore resides in the interaction between retained state, numerical process, and comparator, rather than in predictive accuracy alone.

This distinction changes both model-selection and deployment decisions. Work-based validation improves some held-out comparisons; matched adaptation, stronger nonlinear training, and native repeated correction reveal distinct finite-budget benefits. Independent batches also confirm a 0.73\% complete online saving for one frozen physics-FNO deployment against zero–Poisson-PCG, while other workloads favor classical reuse. These findings separate two decisions: whether learning helps a specified computation, and whether that computation is worth deploying. The first requires a matched state replacement; the second additionally charges the full computational path and grants competent alternatives the reuse their workload permits. A useful learned component and a competitive numerical solver are separate achievements. Correctness validates the output; the value of learning is established by what changes when its retained state is replaced.

\section{Future Work}

The next target is the generalization of computational advantage, not only of predictions. Building on solver-aware training, future studies should test whether learning to remove error components that remain expensive for a chosen corrector yields benefits that persist across changes in coefficients, discretization, preconditioning, and tolerance. Compact, reference-free features from operator actions or short residual trajectories could support cost-aware model selection. Such policies should be evaluated against a competent classical fallback, charging every diagnostic and discarded candidate, with explicit bounds on additional cost and acceptance risk. The central decision problem is when information about a model’s usefulness is worth the computation required to obtain it.

A second direction is to optimize reuse over a complete lifecycle budget. Source-data generation and pretraining, target adaptation and validation, and classical factorization or preconditioner reuse should be compared as competing investments in future solves. Experiments should measure startup and accepted online paths alongside these investments, and vary source datasets, optimization seeds, validation draws, and execution environments. This would turn a gain conditional on one fitted state into a test of whether the training-and-selection procedure reliably produces useful states. Query volume and permitted classical reuse should enter that decision explicitly: at unchanged query costs, repeated deployment amplifies an online advantage but cannot amortize an online disadvantage.

Finally, state interventions should follow the native mechanism and the intended outcome. A direct extension is matched pretrained-versus-scratch target-residual optimization, replacing every source-dependent object—including normalization, anchors, and optimizer state—while holding target information and selection opportunities fixed. Beyond unique-root forward problems, nonlinear branch selection and inverse-posterior calibration should become explicit evaluation targets rather than consequences inferred from residual acceptance. These settings would test whether learning changes which outcome is obtained, as well as the resources needed to obtain it. The broader objective is to make computational advantage itself the object that must transfer.

\label{page:main-end}

\clearpage
\section*{AI use statement}
AI writing assistants were used to help draft and edit the manuscript. All scientific claims, dataset construction decisions, experimental results, and final text were reviewed and approved by the authors.

\section*{Reproducibility statement}
Appendix~\ref{app:artifact-scope} and the anonymous executable supplement distinguish four capabilities: exact verification of the CG construction; reconstruction of the core main table and independent batch confirmation from raw task-level records; reevaluation using recovered weights and runnable task/solver code; and regeneration and analysis of the frozen-library selection study on new task identifiers. All main-table bootstrap units, batch grouping and run order, new selection choices, acceptance checks, and cost boundaries are explicit. The 48 required deployment checkpoints and 24-state selection subset are included.

\bibliographystyle{iclr2027_conference}
\nocite{*}
\bibliography{references}

\clearpage
\appendix
\section{Reading guide and preservation of evidence}
\label{app:guide}
The main text separates three questions: whether an output is accepted, what retaining source state changes under a matched deployment, and whether the resulting system has a complete-cost advantage over competent alternatives. The supplementary sections retain the original study's full experimental scope, including unfavorable configurations, failed runs, exploratory choices, independent confirmations, and incomplete reproducibility boundaries. No historical data are reclassified as new experiments; the separately frozen new-task model-selection study is explicitly identified.

\begin{center}
\begin{tabularx}{\linewidth}{p{0.49\linewidth}Y}
\toprule
Question & Supplementary location \\
\midrule
What exactly do the algebraic statements establish? & Appendix~\ref{app:identification}; the supplied exact-arithmetic verification script.\\
What happened to the earlier three diagnostic ideas? & Appendix~\ref{app:diagnostics}: explicit analytic reconstructions, not additional trained-model experiments.\\
How are tasks, models, correctors, and budgets defined? & Appendices~\ref{app:units}--\ref{app:nonlinear-implementation}, followed by extension protocols.\\
Where are all deployment conditions? & Appendix~\ref{app:deployment-details}; separate batch confirmation in Appendix~\ref{app:batch-confirmation}.\\
How are source count and target information isolated? & Appendices~\ref{app:source-information} and \ref{app:information}.\\
Which native mechanisms and nonlinear extensions were tested? & Appendices~\ref{app:nonlinear-long}, \ref{app:hints-official}, \ref{app:nows-scaled}, and \ref{app:boundary-interface}.\\
How does the literature relate, and what was not tested? & Appendix~\ref{app:extended-literature}; versioned protocol records remain separate from the broad bibliography.\\
Does choosing by work help on fresh tasks? & Appendix~\ref{app:selection-study}; frozen library, new tasks, ranking reversals, and costs.\\
What is actually inside this package? & Appendices~\ref{app:evidence-map} and \ref{app:artifact-scope}.\\
\bottomrule
\end{tabularx}
\end{center}
The historical aggregate blocks are numerical companions, not counts of independent trials. The supplementary text retains source-record distinctions; host labels are anonymized without pooling timings across machines. Original administrative identifiers belong in the author-only archive, not the anonymous submission.

\clearpage

\section{Controls and the claims they support}
Table~\ref{tab:claims} summarizes the controls needed for each claim.
\begin{table}[htbp]
\centering
\small
\caption{\textbf{The control determines the claim.} All paired rows require declared target information, budgets, and evaluation rules. These are complementary comparisons, not interchangeable proofs of neural utility.}
\label{tab:claims}
\begin{tabularx}{\linewidth}{p{0.20\linewidth}YY}
\toprule
Claim & Required observation or comparison & Not established by that observation alone \\
\midrule
Output correctness & Independent acceptance and the stated reference/continuum checks & A benefit attributable to source learning \\
Predictive reuse & Frozen prediction versus an appropriate source-independent prediction control & Reduced numerical correction work \\
Finite-budget adaptation value & Pretrained versus same-architecture scratch, with matched labels and selection opportunities & An irreducible asymptotic error advantage \\
Initialization value & Learned versus zero/coarse under the same corrector and tolerance & Advantage under another corrector or lower total runtime \\
Complete solver utility & Accepted full paths versus competent alternatives, with matched workload and reuse permissions & Lifecycle superiority when offline costs are unaccounted for \\
\bottomrule
\end{tabularx}
\end{table}

\section{Exact constructions for the evidence gaps}
\label{app:identification}
This section uses classical CG to construct a four-dimensional example with identical initial solution, energy, and residual summaries but different correction work at a fixed nonzero tolerance. We derive the exact residual sequences and identical accepted outputs, then establish a strict-dominance corollary.

\subsection{The observation must be specified}
Let $\mathcal X$ be a set of executions with a fixed task and declared algorithmic contract. An observation map $O:\mathcal X\to\mathcal Y$ is sufficient to determine work $K$ on this set only if some map $g$ satisfies $K=g\circ O$. If two executions have equal observations but different work, no such $g$ exists. This is the elementary meaning of ``does not determine'' used in the paper; it is not a statistical identification theorem for an unrestricted causal model.

Reporting that both executions are accepted gives at most a binary observation for each. Reporting the exact limiting solution also gives no work distinction when both converge to the same unique root. These observations need not describe a full finite trajectory. In particular, for $x_k=q^k x_0$ with known $0<q<1$ and finite known $k$, $x_0=q^{-k}x_k$ is recovered from an exact observation. Closeness of outputs alone does not imply zero mutual information, indistinguishability under arbitrary measurement, or loss of all information about source state. Any noise- or quantization-based theorem would have to specify the observation channel and has not been established here.

\subsection{Proof of Proposition~\ref{prop:summary}}
The system is $A=\operatorname{diag}(1,2,3,4)$, $b=\mathbf1$, so $u^\star=(1,1/2,1/3,1/4)^\top$ and $\norm b_2^2=4$. With $e_0=a$ or $v$, the squared spectral weights are
\[
 w=(1/2,0,1/2,0),\qquad z=(1/3,1/2,0,1/6).
\]
For each vector, direct substitution yields
\[
 \sum_i w_i=\sum_i z_i=1,\qquad
 \sum_i\lambda_iw_i=\sum_i\lambda_iz_i=2,\qquad
 \sum_i\lambda_i^2w_i=\sum_i\lambda_i^2z_i=5.
\]
Thus the three squared summaries coincide. For CG with $r_0=b-Au_0=-Ae_0$ and $p_0=r_0$, the exact recurrence is
\begin{align*}
 \alpha_k&=\frac{r_k^\top r_k}{p_k^\top A p_k}, &
 u_{k+1}&=u_k+\alpha_kp_k, &
 r_{k+1}&=r_k-\alpha_kAp_k,\\
 \beta_k&=\frac{r_{k+1}^\top r_{k+1}}{r_k^\top r_k}, &
 p_{k+1}&=r_{k+1}+\beta_kp_k.
\end{align*}
Applying this recurrence gives the following exact squared residuals:
\begin{center}
\begin{tabular}{lrrrr}
\toprule
Initial error & $\norm{r_0}_2^2$ & $\norm{r_1}_2^2$ & $\norm{r_2}_2^2$ & $\norm{r_3}_2^2$ \\
\midrule
$a$   & $5$ & $45/196$ & $0$ & already terminated \\
$v$   & $5$ & $2/3$ & $40/867$ & $0$ \\
$v/2$ & $5/4$ & $1/6$ & $10/867$ & $0$ \\
\bottomrule
\end{tabular}
\end{center}
Acceptance at relative residual $10^{-2}$ is equivalent to squared absolute residual at most $4\times10^{-4}$. For $a$, the first residual exceeds this threshold and the second is zero. For $v$, the first and second exceed it and the third is zero. The SPD system has a unique solution, so the returned vectors in both executions are exactly $u^\star$. Therefore the observation
\[
 O=(A,b,\epsilon,S_A(e_0),u_{K_\epsilon})
\]
is identical for the two executions, but $K_\epsilon$ is two versus three. Even this observation, stronger than terminal acceptance alone, cannot determine work on the constructed pair. Subtracting from the same reference execution's work also cannot determine the corresponding saving from $O$ alone. The example demonstrates unequal savings, not opposite signs against every possible reference.

\subsection{Proof of Corollary~\ref{cor:dominance}}
Scaling $v$ by $1/2$ scales all three squared initial summaries by $1/4$. In exact-arithmetic CG the error and residual sequences scale by the same factor, while the nonzero-step recurrence coefficients remain unchanged. The residual row for $v/2$ in the table follows. At step two,
\[
 \frac{\norm{r_2}_2}{\norm b_2}
 =\sqrt{\frac{10/867}{4}}=\sqrt{\frac5{1734}}>10^{-2};
\]
step three is exact. Thus strict componentwise improvement in solution, energy, and residual summaries is not a sufficient condition for fewer accepted CG updates at this fixed tolerance. This is a worst-case logical counterexample; no frequency in trained networks is inferred.

\subsection{The physical-objective connection and its limits}
The energy identity follows from substituting $u=u^\star+e$ and $Au^\star=b$ into $\mathcal E(u)$; the cross term vanishes. A common positive normalization of each objective across the two initializations preserves equality and strict ordering. This does not identify every PINN loss with $S_A$: collocation, quadrature, boundary penalties, parameterization, and non-SPD operators can alter the objective.

For a differentiable parameterization $u_\theta$ on the same fixed linear system, let $r_\theta=Au_\theta-b$, $J_\theta=\partial u_\theta/\partial\theta$, and $\mathcal L=\tfrac12\norm{r_\theta}_2^2$. Continuous Euclidean gradient flow $\dot\theta=-\nabla_\theta\mathcal L$ gives, by the chain rule,
\[
 \dot r_\theta=-AJ_\theta J_\theta^\top A^\top r_\theta.
\]
This exact identity illustrates why a scalar residual value alone does not specify residual dynamics: the parameterization also enters. It is not a fixed-kernel approximation or a convergence guarantee for Adam, L-BFGS, sampled continuous PINNs, or arbitrary nonlinear residuals. Related PINN kernel and operator-conditioning analyses are prior work \citep{ntk2020,precondition2023}. The present experiments do not add a coordinate-PINN optimization benchmark.

\subsection{Scope of the constructions}
The constructions show that the selected scalar summaries and accepted output can coincide while correction work differs under a fixed finite-dimensional solver contract. The actual-matrix oracle interventions extend this analysis to the tested PDE discretizations by constructing errors with matched summaries and different spectral distributions. These results establish the insufficiency of the summaries for determining work. Explaining the trained-model effects requires further evidence linking their error structure to correction dynamics; the prevalence of these effects and their prediction from inexpensive features remain open questions.
\clearpage

\section{Earlier diagnostic ideas retained as analytic sanity checks}
\label{app:diagnostics}
The following analytic constructions examine three distinct gaps: sampled residual versus continuous residual, residual magnitude versus solution error, and corrected accuracy versus the contribution of a retained predictor. Together, they motivate the controlled trained-network experiments.

\subsection{Zero sampled residual with unit relative solution error}
On $(0,1)$, consider $-u''=\pi^2\sin(\pi x)$ with zero Dirichlet data. The unique solution is $u^\star(x)=\sin(\pi x)$. For integer $m\geq2$, define $\widehat u(x)=\sin(\pi x)+\sin(m\pi x)$ and sample at $x_j=j/m$, $j=1,\ldots,m-1$. The residual is
\[
 -\widehat u''(x)-\pi^2\sin(\pi x)=m^2\pi^2\sin(m\pi x),
\]
which vanishes at every sampled point and at the boundary in exact arithmetic. Orthogonality of sine modes gives $\norm{\widehat u-u^\star}_{L^2}/\norm{u^\star}_{L^2}=1$. The continuous residual is not small. Floating-point sine evaluation gives small rounding residuals rather than a symbolic zero.

\subsection{A well-posed problem with very small relative residual and unit error}
Let $\delta=10^{-8}$ and $L_\delta=-d^2/dx^2-\pi^2+\delta$ on $H^2(0,1)\cap H_0^1(0,1)$, measured in $L^2$. Its eigenvalues are $(k^2-1)\pi^2+\delta>0$ for $k\geq1$, so the problem is invertible and the inverse $L^2$ operator norm is $1/\delta$. Set
\[
 u^\star=\sin(2\pi x),\quad
 f=(3\pi^2+\delta)\sin(2\pi x),\quad
 \widehat u=u^\star+\sin(\pi x).
\]
Then
\[
 \frac{\norm{\widehat u-u^\star}_{L^2}}{\norm{u^\star}_{L^2}}=1,
 \qquad
 \frac{\norm{L_\delta\widehat u-f}_{L^2}}{\norm f_{L^2}}
 =\frac{\delta}{3\pi^2+\delta}<10^{-9}.
\]
Well-posedness does not make the stability constant uniformly favorable as $\delta\to0$. 

\subsection{Correcting an uninformative learned map}
For an SPD matrix $A$, let $\mathcal S$ be a proper invariant eigenspace and $P_{\mathcal S}$ its orthogonal projector. A retained linear map $G=A^{-1}P_{\mathcal S}$ solves the system exactly on right-hand sides in $\mathcal S$ but returns zero on $\mathcal S^\perp$. For a nonzero target $b\in\mathcal S^\perp$, the learned prediction $Gb=0$ equals the zero control and has unit relative solution error. CG nevertheless solves the SPD system from either start in at most the matrix dimension in exact arithmetic. Final exactness therefore does not distinguish the uninformative retained map from the control. For a mixed right-hand side, $G$ removes the familiar component; the remaining Krylov work depends on the unrepresented spectral support. 

\paragraph{Relation to the main experiment.}
These examples separate sampling reliability, the stability of residual-based error estimates, and the contribution of a retained predictor to corrected accuracy. The equal-moment construction in Appendix C addresses the next question: whether knowing the initial solution error, energy error, and residual is sufficient to determine correction work. The empirical experiments examine these distinctions in trained models.
\clearpage

\section{Scope, provenance, and independent units}
\label{app:units}
The core experiments are independent controlled implementations, not a re-creation of every experiment in the cited papers. Original uploaded aggregate tables cannot recover missing per-task results or model weights; we do not pool them into the newly executed samples. The initial controlled wave and its 6000-update expansion have separately retained records. Both positive and unfavorable configurations are available. An endpoint means one task, discretization, state, corrector, tolerance, and interface combination; a physical task means one sampled equation instance. Neither optimizer repetitions nor measurements on multiple grids increase the number of physical tasks.

\begin{table}[h]
\centering\small
\caption{Completed core waves before the deployment and nonlinear extensions. These are counts of records or trajectories, not interchangeable sample sizes.}
\begin{tabular}{lrrr}\toprule
Quantity & Initial wave & Expanded wave & Total\\\midrule
Source-training trajectories &48&192&240\\
Target-adaptation trajectories &36&108&144\\
Main correction endpoints &34368&460224&494592\\
Independent main tasks &256&128&384\\
Timing repetitions &2184&37440&39624\\
Spectral records &64&41463&41527\\\bottomrule
\end{tabular}
\end{table}
The expanded wave also uses 256 adaptation-test tasks, 64 timing tasks, and 32 spectral tasks, separate from its 128 main tasks. Thus it introduces 480 test-task identities. Source dataset identifiers are separate from optimizer seeds. Within a source dataset, paired-task summaries average optimizer repetitions before resampling physical tasks. Two source datasets are reported separately, not treated as sufficient draws for a population-level confidence statement about training data. Expanded-wave intervals use 4000 resamples (initial-wave archived analyses used 10000); they describe the fitted configurations under their specified task distribution, and are unadjusted for multiple comparisons. They are not simultaneous confidence guarantees over the entire audit.

\section{Linear PDE generators and numerical implementation}
\label{app:linear}
\subsection{Equations and source distributions}
The domain is $[0,1]^2$ with zero Dirichlet boundary conditions. Diffusion solves $-\nabla\cdot(a\nabla u)=f$; advection--diffusion solves $-0.1\nabla\cdot(a\nabla u)+(2,1)\cdot\nabla u=f$. A finite trigonometric series $g$ uses frequencies $(i,j)\in\{-4,\ldots,4\}^2\setminus\{(0,0)\}$, independent Gaussian sine/cosine coefficients, and weights $\exp(-2\pi^2\ell^2(i^2+j^2))$. Its coefficient variance normalization is shared across resolutions. We set $a=\exp(\sigma\tanh g)$. In distribution, $\sigma\sim U[0.5,1]$ and $\ell\sim U[0.15,0.30]$; contrast shift uses $\sigma\sim U[1.5,2]$, roughness shift uses $\ell\sim U[0.05,0.10]$, and joint shift changes both. Forcing is one plus a normalized random sine series over modes $1$--$4$ in both directions, scaled by $0.4$ and with inverse-frequency coefficient decay. The historical run record specifies the exact generator and random draw order; this editorial package does not include every original data array (Appendix~\ref{app:artifact-scope}). Re-evaluating the continuous series on another grid does not resample a task.

The interior grid has $n\in\{31,63,127\}$ and spacing $h=1/(n+1)$. Diffusion fluxes use harmonic face coefficients; positive-velocity advection uses first-order upwind differences. Matrix assembly and the differentiable training operator implement the same discrete equation. A sparse direct solve provides an offline reference; it is not passed to a deployment initializer, selector, or iterative stopping rule. Manufactured-solution refinement, matrix/tensor-operator agreement, boundary values, and finite gradients are checked before scientific runs. These checks do not turn a finite-dimensional reference into an exact continuous solution.

\subsection{Networks and optimization}
The base FNO has width 16, three spectral blocks and six retained Fourier modes per axis, pointwise mixing, GELU activations, and a width 32 projection. Inputs are $\log a,f,x,y$. The base CNN has width 56 and convolutional dilation sequence 1,2,4,8,1, spanning the training domain. The output multiplies a hard zero-boundary factor proportional to $x(1-x)y(1-y)$, common to both objectives. The wide controls use FNO width 32 and CNN width 112. Real parameter accounting counts each complex parameter as two real scalars.

The data objective is mean squared relative discrete solution error. A physics objective adds $\lambda$ times mean squared relative discrete residual; $\lambda\in\{10^{-4},10^{-3},10^{-2}\}$ in the principal sweep. All arms share supervised source labels, so this is not a claim of label-free training. Adam uses batch 24 and gradient clipping 1. The expansion runs 6000 updates with a common cosine learning-rate schedule from $10^{-3}$ to $10^{-4}$. Checkpoints at 1200,3000,6000 retain the best source-validation relative error so far. These are prefixes of one training schedule, not three independently optimized short schedules. Positive-$\lambda$ selection excludes the data-only arm and uses no target-test labels.

The 96-trajectory sweep is two equations times two source datasets times two architectures times four weights times three optimizer seeds. The 48 wider controls use two objectives with fixed $\lambda=10^{-3}$ for physics. The 48 new-source controls use 1024 training and 128 validation examples, whereas source datasets 0/1 use 384/64. The latter block changes both data realization and sample count, and cannot by itself identify the effect of sample count.

\subsection{Correction, padding, and acceptance}
For diffusion we test CG, Poisson-PCG, and symmetric smoothed-aggregation AMG-PCG. The constant-coefficient Poisson inverse uses orthonormal discrete sine transforms and exact discrete eigenvalues. AMG uses symmetric Gauss--Seidel smoothing and a V-cycle. For advection--diffusion we test restarted GMRES, Poisson-GMRES, and ILU-GMRES, with ILU drop tolerance $10^{-4}$ and fill factor 10. GMRES restarts after 50 updates; both solver families have a total 1200-update cap. Every accepted endpoint must satisfy the true float64 relative residual at tolerance $10^{-4},10^{-6}$, or $10^{-8}$, with multiplier $1+10^{-6}$ for numerical slack. Callback or preconditioned residuals alone do not establish acceptance.

The coarse control solves the same task on $n_c=(n-1)/2$ and interpolates with aligned boundary nodes; its setup and solve are charged in time comparisons. Direct, zero, and coarse controls are not trained. An exact zero right-hand side uses the separately declared absolute-residual rule in the deployment extension.

With $N=n+2$ full-grid points, the FNO padding rules are fixed $p=7$, fractional $p=\operatorname{round}(7(N-1)/32)$, and fixed-period $p=\operatorname{round}(40(N-1)/32)-N$. They coincide on the source grid but differ on refined grids. Only predeclared variants are compared; a favorable target-grid result is not used to pick a deployment interface retroactively. Fine-grid subsets contain 8/4 tasks per stratum at $n=63/127$, compared with 16 at $n=31$; common-grid conclusions use their intersection.

\begin{table}[h]
\centering\small
\caption{All failed expanded-wave main endpoints. All other family/grid/corrector/tolerance groups pass; failures remain in their original denominators.}
\begin{tabular}{llrrr}\toprule
Equation & Corrector & Tolerance & Failed & Total\\\midrule
Advection--diffusion & GMRES & $10^{-8}$&953&2336\\
Advection--diffusion & GMRES & $10^{-6}$&476&2336\\
Advection--diffusion & GMRES & $10^{-4}$&132&2336\\
Diffusion & CG & $10^{-8}$&1009&2336\\
Diffusion & CG & $10^{-6}$&21&2336\\\bottomrule
\end{tabular}
\end{table}
These groups all use $n=127$. The 2591 failed endpoints are capped outcomes, not missing observations. Maximum relative discrete errors among accepted endpoints at tolerances $10^{-8},10^{-6},10^{-4}$ are respectively $1.03936\times10^{-8}$, $1.12663\times10^{-6}$, and $1.16954\times10^{-4}$. A residual threshold is not an identical relative-solution-error threshold.

\section{Adaptation and target-information accounting}
\label{app:adaptation}
The expanded adaptation study starts from the initial-wave source FNO checkpoints, not the newly trained 6000-step states. Both PDE families use nested target-training counts 16/64/256, a shared 64-example validation set, and 128 independent test tasks. Source 0 and source 1 pretrained arms each use three optimizer seeds; a source-independent scratch arm is trained once per family, training-sample count, learning rate, and optimizer seed and is not duplicated for the source comparisons. Learning rates are $10^{-3}$ and $3\times10^{-4}$. Validation every 25 updates, including step 0, selects the best-so-far checkpoint at maximum budgets 200/800/1600 and then selects the learning rate. Test data do not choose either. The 108 actual trajectories give multiple dependent budget estimates.

The frozen source predictor is a separate control, also available as step 0 for pretrained validation selection. Scratch and pretrained arms use identical target observations and validation opportunities. Reporting 16 training examples without the 64 validation labels would conceal part of the allowed target information. Relative to the earlier fixed-step wave, changed test tasks and validation size prevent attributing cross-wave improvement solely to early stopping. Within-trajectory information/stopping interventions are listed separately from this historical comparison.

\section{Spectral constructions and their limits}
\label{app:spectral}
\subsection{An explicit three-norm counterexample}
Let $A=\operatorname{diag}(1,2,3,4)$ and let two error vectors have squared coordinates $w=(1/2,0,1/2,0)$ and $v=(1/3,1/2,0,1/6)$. Their squared Euclidean, energy, and residual norms are respectively $\sum w_i$, $\sum\lambda_iw_i$, and $\sum\lambda_i^2w_i$. Direct substitution gives $(1,2,5)$ for both vectors. In exact arithmetic a polynomial $p$ with $p(0)=1$ must vanish at each eigenvalue carrying nonzero initial error to give exact convergence. The first support has two distinct eigenvalues; the second has three. Positive spectral weights rule out earlier exact convergence, so CG takes two versus three updates. This elementary construction applies to exact termination; finite-tolerance and finite-precision behavior additionally depend on the stopping rule and arithmetic \citep{cg1952,saad2003,carson2024}.

\subsection{Actual-matrix oracle interventions}
The 32 new joint diffusion tasks at $n=31$ use all 12 initial-wave FNO states (two source datasets, three optimizer seeds, two objectives). We diagonalize either $A$ or the symmetric Poisson-transformed operator. Controls preserve Euclidean spectral mass while changing direction, reverse spectral placement, or change coefficient signs. Three-moment linear programs constrain nonnegative spectral weights to preserve $\sum w_i$, $\sum\lambda_iw_i$, and $\sum\lambda_i^2w_i$, while maximizing or minimizing a higher-order diagnostic. Basic feasible solutions concentrate mass on at most three spectral atoms. Rapid Krylov termination is therefore expected; this construction identifies insufficiency of the three summaries rather than discovering a practical initializer.

Exact references and full eigendecompositions are used only in this oracle diagnostic. They must not be used as deployable selection features. Factors 0.5/1/2 and the three tolerances yield 41472 planned records. One intervention fails its numerical moment-matching gate (relative defect about 0.2006), so its nine follow-ups are not executed; the actual total is 41463. This is recorded as a rejected numerical construction, not a claim that the mathematical constraints have no feasible solution. Successful matching has maximum relative defect below $4.12\times10^{-13}$. The sign control preserves exact spectral magnitudes but exhibits small finite-precision differences; in transformed PCG coordinates it also need not preserve the physical Euclidean stopping residual. More explicitly, for an SPD preconditioner $P$, let $B=P^{-1/2}AP^{-1/2}$ and $z=P^{1/2}e$. Then $\|z\|_2^2=e^\top Pe$, $\|z\|_B^2=e^\top Ae$, and $\|Bz\|_2^2=r^\top P^{-1}r$. Only the energy quantity is necessarily the physical $A$-energy error; changing transformed spectral signs can change the physical residual even in exact arithmetic. The corresponding arguments do not transfer to nonnormal restarted GMRES, whose behavior has additional structure \citep{greenbaum1996}.

\section{Nonlinear implementation and a standard stability bound}
\label{app:nonlinear-implementation}
For $F_h(u)=\nu L_hu+\gamma u^{\odot3}-f$ with zero Dirichlet data, $\nu>0$ and $\gamma\geq0$, the exact Jacobian is $\nu L_h+3\gamma\operatorname{diag}(u^2)$. The banded solve is used inside residual-norm Armijo-damped Newton with 40 updates, at most 26 line-search attempts per update, halving, and Armijo constant $10^{-4}$. A zero-initialized strict reference instead uses 100 updates and tolerance $10^{-11}$. Independent adaptive continuous collocation uses residual tolerance $10^{-7}$, boundary tolerance $10^{-10}$, and at most 50000 nodes; it is a numerical crosscheck. Refined-grid comparisons use $2n+1$ interior points.

\begin{proposition}[Standard discrete strong-monotonicity estimate]
For the one-dimensional Dirichlet second-difference operator and $h=1/(n+1)$, $F_h$ is strongly monotone with $m_h=4\nu h^{-2}\sin^2(\pi h/2)$. If $u_h^*$ is its zero, then $\norm{u-u_h^*}_2\leq\norm{F_h(u)}_2/m_h$.
\end{proposition}
\begin{proof}
For $d=u-v$, $d^\top(F_h(u)-F_h(v))=\nu d^\top L_hd+\gamma\sum_i(u_i-v_i)(u_i^3-v_i^3)$. Each term in the sum is nonnegative and the smallest eigenvalue of $L_h$ is $4h^{-2}\sin^2(\pi h/2)$. Thus the expression is at least $m_h\norm{d}_2^2$. Set $v=u_h^*$ and use Cauchy--Schwarz; the case $d=0$ is immediate. Strict convexity and coercivity of the discrete potential $\nu u^\top L_hu/2+\gamma\sum_i u_i^4/4-f^\top u$ also establish a unique root.
\end{proof}
The estimate deteriorates as $\nu$ tends to zero and controls distance to a discrete root, not the continuum discretization error. If a numerical reference $v$ has nonzero residual, the same proof gives $\norm{u-v}\leq(\norm{F_h(u)}+\norm{F_h(v)})/m_h$. We do not treat a finite-precision reference as an exact zero when comparing numerical bounds. Nonmonotone reaction terms or multiple-solution branch selection require different assumptions and experiments.

The nonlinear source distribution has $\nu\in[0.05,0.2]$, $\gamma\in[0.5,2]$ sampled logarithmically, an offset forcing in $[-0.5,0.5]$, amplitude in $[0.5,1.5]$, and a normalized sine series with modes 1--8 and $k^{-2}$ decay. Shift strata use $\nu\in[0.002,0.01]$, $\gamma\in[8,32]$, modes 1--24 with $k^{-0.25}$ decay, their joint change, the linear limit $\gamma=0$, and an extreme stratum $\nu\in[5\times10^{-5},2\times10^{-4}]$, $\gamma\in[32,128]$ with rough forcing. Each has 16 physical tasks. The one-dimensional FNO uses width 24, modes 12, depth 3, right padding of one quarter-domain, exact zero boundary values, and input forcing, $\log(\nu/0.1)$, $\log(1+\gamma)$, and coordinate. The 512/96 source pool is common across three seeds and two objectives. Adam uses batch 32,1200 cosine-schedule updates, and fixed physics weight $10^{-3}$. Validation chooses the checkpoint; this is not an exhaustive physics-weight search.

\section{Cost accounting and incomplete runs}
\label{app:cost}
The original serial cost suite uses old source 0/seed 11 data and physics models,32 tasks per family, three grids, three tolerances, thirteen paths, and five randomized repetitions after warmup. This produces 37440 timing records. Timing includes assembly, encoding, transfers, inference, requested preconditioner setup, correction, and true-residual verification, and excludes field sampling, offline truth, training, and model loading. Two of 192 task-grid units were recovered on another allocation after the original one-hour time limit. Their measurements remain separately identifiable. The 30-task versus32-task sensitivity comparison is retained; changes in task composition and node cannot be disentangled by that comparison.

An expectation of total cost uses means, not an unqualified sum of medians. Table entries obtained by averaging within-task repetition medians are descriptive timing summaries; query-count curves formed from those measurements are scenario projections rather than unbiased expected-cost estimators.

Formal time-to-solution summaries require all prescribed repetitions to pass the corresponding acceptance rule. A median over successful cases alone is separately labeled and is not substituted for a full-workload timing estimate. Warm resident, process-cold, variable-matrix batches, and cached fixed-matrix solves answer different questions; their denominators and timing boundaries must not be pooled. GPU inference with a CPU sparse corrector does not establish the optimal implementation of either classical or learned methods. Training cost, setup cost, per-query cost, and failure probability are separately relevant to a lifecycle comparison.

\section{Coverage ledger and extension protocols}
\label{app:extensions}
The accompanying coverage ledger records execution and verification separately. The deployment, target-information, nested source-count, and nonlinear long-budget protocols below have complete result receipts, reported in the results appendices. The scaled NOWS branch also has a complete evaluation receipt and is reported in Appendix~\ref{app:nows-scaled}. Boundaries outside these protocols remain explicitly untested.

\paragraph{6000-update states and deployment workloads.}
The new warm protocol uses 24 source-selected states per family: two source datasets, three seeds, FNO/CNN, and data/positive-residual objectives at 6000 updates. Twelve independent joint tasks per family are paired across grids 31/63/127 with three blocked repetitions. Primary correction is Poisson-PCG or Poisson-GMRES; direct, zero, coarse, and zero with AMG/ILU supply controls. A 255-grid extension uses four new tasks per equation and source 0/seed 11 states. Matrices, encodings, transfers, initialization, setup, correction, and true residual checks are charged.

Variable-matrix batch 8 and batch 32 share the first eight task identities and use genuine batched neural inference. Two independent batches per family/grid limit precision of any throughput conclusion. Fixed-matrix workloads use four matrices per family/grid and eight right-hand sides per matrix, allowing classical direct-factorization, coarse-factorization, and preconditioner reuse. Setup is measured separately. $Q\in\{1,10,100,1000\}$ amortization curves are projections from measured costs, not 1000 physically observed solves. Within-process weight reloading and a separate fresh-interpreter/CUDA-process experiment distinguish two kinds of cold start; the latter also includes field generation and does not flush operating-system caches.

\paragraph{Input interventions and deployable selectors.}
Eight common base fields per family undergo joint shift, coefficient-log multipliers 0.25/2/3, forcing multipliers 0.01/100/$-1$, and zero forcing. These interventions are paired, not 64 independent fields. An exact-zero right-hand side uses absolute residual tolerance $10^{-12}$. Selectors compare initial residual or five-update probe performance among zero, coarse, and a data FNO, without a reference solution. Every discarded candidate and probe is charged; continuation updates alone are not total work. Thirty-two independent target tasks per family are reserved for this fixed policy. It is a bounded selector comparison, not a fitted predictor of universal neural usefulness.

\paragraph{Equal information and source-data control.}
One 80-label target pool per equation is split 16/64, 40/40, or 64/16 into training/validation; tasks can change role across splits but never overlap within a split. Two pretrained source states and scratch, three seeds, two learning rates, and both equations produce 108 actual trajectories. Fixed-final and best-validation checkpoints are saved from the same trajectories at 200/800/1600 updates. Within-learning-rate contrasts isolate checkpoint policy; independently validation-selected rates compare a broader policy. 128 test tasks per equation are kept separate from all selection.

A source-count extension uses the exact nested 128/384/1024 prefix of one source pool, fixed 64 validation tasks, and fixed 6000 updates, with two equations, three seeds, and data/residual objectives. Twelve 384-sample arms reuse existing checkpoints after exact configuration checks; 24 arms are new trajectories. This removes the original data-realization/count confounding for one pool, but holding updates fixed changes expected sample exposures. It is not a universal data-scaling law.

\paragraph{Scaled official NOWS branch.}
The independent wrapper imports the official Darcy FNO and VINO training function \citep{nows2026}, retaining width 32, 32 modes, four blocks, the energy objective, zero boundary ansatz, and official finite-difference matrix. It reduces grid 256 to 64, training samples 1000 to 256, and epochs 2000 to 100. Three seeds share one source pool; source references are logged but do not enter the energy gradient or final-epoch selection. ID, contrast, and roughness each reserve 16 new tasks, grids 64/128, two tolerances, and direct/zero/coarse/learned correction paths. This is a reduced-budget native-code audit, not a full-paper replication. All three training states, bridge checks, and 96 task/grid units are complete; the 9216 formal timing records retain their host identities.

\paragraph{Boundaries outside this empirical matrix.}
Geometry transfer, unstructured meshes, same-weight cross-equation transfer, nonmonotone branch selection, long-time rollouts, calibrated inverse posteriors, and language-model-generated solvers are represented in the literature map. They require different task and error definitions; neither linear correction experiments nor the unique-root nonlinear extension tests them. Their absence is a scope limitation rather than a negative empirical result.

\section{Versioned protocol records from the earlier evidence audit}
\label{app:protocol-records}
The following 18 records from 12 papers preserve the earlier targeted extraction and its version-specific locators. They are separate from the expanded 104-record bibliographic coverage. They describe controls in the cited papers. Updated publication metadata do not silently change the version underlying a protocol locator.
\begingroup\small\setlength{\tabcolsep}{5pt}
\begin{longtable}{@{}>{\raggedright\arraybackslash}p{.065\linewidth}>{\raggedright\arraybackslash}p{.285\linewidth}>{\raggedright\arraybackslash}p{.285\linewidth}>{\raggedright\arraybackslash}p{.285\linewidth}@{}}
\caption{Versioned, targeted full-text protocol records. Controls described here belong to the cited papers.}\label{tab:ledger}\\
\toprule Record & Source locator & Retained object / target process & Existing evidence and scope \\ \midrule\endfirsthead
\toprule Record & Source locator & Retained object / target process & Existing evidence and scope \\ \midrule\endhead
P01 & \citep{pino2021}. 2111.03794v4; Sec.3.1; Table 1; Appendix C & operator weights. coefficient or initial field; prescribed PDE; forward prediction. & data-only versus data-plus-PDE training. not a uniform claim across equations. \\[4pt]
P02 & \citep{pino2021}. 2111.03794v4; Sec.3.2; Appendix C & weights and optional pretrained anchor. target residual and conditions; instance optimization. & predictions before/after optimization. anchor state must be included in a reset. \\[4pt]
P03 & \citep{pino2021}. 2111.03794v4; Table 4 & operator pretraining. target Kolmogorov equation; target optimization. & zero versus nonzero source training budgets. not an all-budget matched-protocol effect. \\[4pt]
P04 & \citep{hints2022}. 2208.13273v2; Sec.2.2; Fig.2; Sec.2.5 & DeepONet correction. target operator and residual; alternating neural and numerical updates. & relaxation; one-time initializer; multigrid comparisons. spectral complementarity is established prior work. \\[4pt]
P05 & \citep{fcg2024}. 2402.05598v1; Sec.3.3; Tables 1--3 & neural nonlinear preconditioner. matrix and residual; flexible conjugate gradients. & classical preconditioners; loss and residual-sampling ablations. iterations are not complete runtime accounting. \\[4pt]
P06 & \citep{fcg2024}. 2402.05598v1; Sec.3.4 & low-resolution learned preconditioner. higher-resolution discrete target; flexible CG. & resolution and architecture comparisons. not a new physical equation family. \\[4pt]
P07 & \citep{metapde2022}. 2211.01604v1; Sec.2--3; Appendix C (Burgers) & meta-learned initialization. target residual; geometry; conditions; gradient-based adaptation. & MAML/LEAP and numerical comparisons. hardware and accuracy regimes condition utility. \\[4pt]
P08 & \citep{learnedsolver2024}. 2410.06820v5; Algorithm 1; Sec.4.3; Fig.2; Appendix F & update network; B-spline representation. PDE parameters and physical-loss gradients; updates of target coefficients. & Adam, SGD and L-BFGS in the same representation. an update count is not a gradient-call count. \\[4pt]
P09 & \citep{poseidon2024}. 2405.19101v2; Appendix C.1--C.2; Table 5 & pretrained ScOT weights. downstream training examples; fine-tuning. & ScOT-from-scratch included. 40M scratch is not size-identical to every scale. \\[4pt]
P10 & \citep{poseidon2024}. 2405.19101v2; Appendix C.1; downstream configurations & pretrained weights; some reinitialized layers. training and validation data; checkpoint selection and adaptation. & lowest-validation-loss checkpoint rule. sample efficiency alone is not a matched time effect. \\[4pt]
P11 & \citep{icon2023}. 2304.07993v3; Algorithm 2; Sec.3.4 & frozen conditional operator. demonstration condition--output pairs; query; contextual forward inference. & number and type of demonstrations. not observation-free prediction. \\[4pt]
P12 & \citep{splitting2026}. 2602.00884v2; Sec.3; Sec.4; Table 1 & operator dictionary. observed target trajectory prefix; operator-composition search. & original dictionary predictor; uniform/beam variants. observations and search remain target resources. \\[4pt]
P13 & \citep{pidm2024}. 2403.14404v4; Eq.14; Sec.4 Darcy experiments & diffusion prior. sampling conditions under the stated task; denoising. & physics-residual training variants. residual consistency does not certify a posterior. \\[4pt]
P14 & \citep{diffusionpde2024}. 2406.17763v2; Sec.3; Eq.8; ablation section & joint diffusion prior. sparse observations; governing equation; observation and PDE-guided sampling. & guidance component ablations. do not assign all accuracy to prior learning. \\[4pt]
P15 & \citep{evopinn2026}. 2607.26490v1; Sec.5.1--5.3; Appendix D & discovered executable algorithm. target PDE; search-time evaluation when discovering; freeze algorithm; retrain solution network. & expert designs and matched search alternatives. discovery and per-task training have separate costs. \\[4pt]
P16 & \citep{evopinn2026}. 2607.26490v1; Sec.5.6; Table 5 & frozen solver program. neighboring PDE parameter; solution network retraining. & seed and expert algorithms. zero-shot refers to algorithm search, not solution training. \\[4pt]
P17 & \citep{williams2024}. 2411.18459v1; Sec.2.4; Sec.3.1--3.2 & extracted trunk basis. initial condition and PDE; spectral coefficient evolution. & data-driven and physics-informed bases. not intrinsic semantic decoding of all weights. \\[4pt]
P18 & \citep{williams2024}. 2411.18459v1; Sec.2.2; Sec.3.2.2--3.3 & initialization from related problems. target physical training constraints; further neural training. & training from related sources. basis quality and target training must be distinguished. \\[4pt]
\bottomrule\end{longtable}
\endgroup

\FloatBarrier
\section{Complete deployment results and sampling units}
\label{app:deployment-details}
The completed warm, batching, caching, weight-reload, boundary and selector audit contains 344 verified units and 15,768 JSON records. A separate fresh-process suite adds 16 units and 240 records, giving 360 units and 16,008 JSON records in total. These include 504 batch records with 10,080 accepted individual queries. Including fresh-process solves, query-expanded accounting gives 25,584 final-query records, 142 of which fail; repetitions and re-discretizations remain dependent. The main serial suite has 12 physical tasks per equation, reported separately by source dataset after averaging three optimizer seeds within task. Other suites have their own clusters: four matrices per equation for caching, two groups per equation and batch size for batching, eight paired base fields per equation for boundary interventions, and 32 tasks per equation for selectors. None of these groups is pooled into a common independent-sample count.

Timings are conditional on this implementation: GPU neural inference, single-thread CPU sparse correction, and true-residual verification. The classical alternative is AMG-PCG for diffusion or ILU-GMRES for advection; every neural and coarse path in these tables uses the primary Poisson corrector. Classical columns repeated across source rows are the same observations. Reported table entries are task means after within-task repetition medians and, where applicable, optimizer-seed averaging; the complete CSVs additionally retain task medians, distributions, failures, and paired intervals. Thus they need not equal medians from the earlier timing wave.

For a cached matrix, setup is charged once over the eight actually measured right-hand sides. Curves at 1, 10, 100, and 1000 queries are projections from measured setup and per-query cost, not newly executed workloads or training-cost break-even estimates. Source data generation, reference labels, training, and selection search are not zero-cost; their complete attribution to a deployment contract is unavailable, so inference-side comparisons do not imply lifecycle superiority.

The residual selector forms zero, coarse, and neural candidates before choosing. The five-update probe performs an independent short solve for each candidate, then restarts the numerical solver from the selected field. It does not continue an unchanged CG search direction or GMRES basis. Reported work and time include all discarded probes plus the restarted solve. Merely including zero among candidates is not a theorem that selected cost or capped acceptance cannot deteriorate.

All 142 failed boundary records occur under coefficient-log scale three. These are limitations of the fixed Poisson-preconditioned, capped paths under this intervention; they do not establish that stronger preconditioners must fail. Direct solution accepts every case. An exact zero right-hand side uses absolute tolerance; zero-field acceptance does not establish neural predictor homogeneity.

\begin{table}[t]
\centering\small\setlength{\tabcolsep}{3pt}
\resizebox{\linewidth}{!}{%
\begin{tabular}{lllllllllll}
\toprule
PDE & $n$ & Source & Direct & Zero & Coarse & Alt. & FNO-D & FNO-P & CNN-D & CNN-P \\
\midrule
Diff. & 31 & 0 & 1.79 & 5.51 & 5.87 & 10.33 & 7.32 & 7.22 & 6.09 & 6.01 \\
Diff. & 31 & 1 & 1.79 & 5.51 & 5.87 & 10.33 & 7.31 & 7.23 & 6.13 & 6.02 \\
Diff. & 63 & 0 & 8.63 & 10.23 & 11.32 & 17.73 & 12.04 & 11.90 & 11.01 & 11.05 \\
Diff. & 63 & 1 & 8.63 & 10.23 & 11.32 & 17.73 & 12.02 & 11.89 & 11.07 & 11.08 \\
Diff. & 127 & 0 & 40.87 & 28.57 & 33.46 & 46.19 & 30.36 & 29.89 & 30.32 & 30.42 \\
Diff. & 127 & 1 & 40.87 & 28.57 & 33.46 & 46.19 & 30.29 & 29.92 & 30.34 & 30.53 \\
Diff. & 255 & 0 & 228.48 & 115.88 & 141.59 & 159.52 & 119.29 & 116.20 & 124.53 & 123.74 \\
Adv. & 31 & 0 & 2.33 & 19.79 & 20.06 & 3.53 & 21.82 & 21.21 & 19.92 & 19.75 \\
Adv. & 31 & 1 & 2.33 & 19.79 & 20.06 & 3.53 & 21.71 & 21.38 & 19.95 & 19.72 \\
Adv. & 63 & 0 & 8.04 & 37.97 & 38.61 & 30.22 & 40.23 & 39.34 & 38.44 & 38.64 \\
Adv. & 63 & 1 & 8.04 & 37.97 & 38.61 & 30.22 & 40.28 & 39.54 & 38.65 & 38.59 \\
Adv. & 127 & 0 & 41.14 & 100.01 & 100.32 & 426.07 & 103.51 & 102.22 & 101.71 & 103.62 \\
Adv. & 127 & 1 & 41.14 & 100.01 & 100.32 & 426.07 & 104.05 & 102.40 & 101.65 & 102.31 \\
Adv. & 255 & 0 & 232.33 & 479.78 & 490.47 & 4446.59 & 480.09 & 472.22 & 489.54 & 492.68 \\
\bottomrule
\end{tabular}}
\caption{Full warm-path milliseconds per accepted query, using source-selected 6000-update states. D/P denote data/positive-residual objectives. Classical controls repeat across source-dataset rows for readability; they are not new runs. At grids 31/63/127, three optimizer seeds are averaged within each of 12 physical tasks. Grid 255 uses four new tasks and source 0, seed 11. Every listed endpoint succeeds. The primary corrector is Poisson-PCG or Poisson-GMRES; the alternate zero control uses AMG or ILU.}
\label{tab:deploy-warm_all_grids}
\end{table}

\begin{table}[t]
\centering\small\setlength{\tabcolsep}{3pt}
\resizebox{\linewidth}{!}{%
\begin{tabular}{llllllllll}
\toprule
PDE & $n$ & Batch & Direct & Zero & Coarse & FNO-D & FNO-P & CNN-D & CNN-P \\
\midrule
Diff. & 31 & 8 & 1.69 & 5.54 & 5.83 & 5.73 & 5.59 & 5.48 & 5.37 \\
Diff. & 31 & 32 & 1.66 & 5.32 & 5.65 & 5.27 & 5.26 & 5.17 & 5.10 \\
Diff. & 63 & 8 & 7.27 & 10.43 & 11.32 & 10.63 & 10.47 & 10.69 & 10.55 \\
Diff. & 63 & 32 & 7.26 & 10.04 & 11.03 & 10.12 & 9.91 & 10.27 & 10.21 \\
Diff. & 127 & 8 & 39.56 & 29.42 & 33.63 & 29.70 & 29.23 & 30.91 & 30.97 \\
Diff. & 127 & 32 & 40.59 & 28.81 & 33.09 & 29.26 & 28.63 & 30.49 & 30.17 \\
Adv. & 31 & 8 & 2.19 & 21.43 & 21.87 & 21.79 & 21.32 & 21.04 & 20.62 \\
Adv. & 31 & 32 & 2.17 & 21.20 & 21.68 & 21.19 & 20.69 & 20.56 & 20.20 \\
Adv. & 63 & 8 & 7.91 & 41.97 & 42.42 & 42.69 & 41.85 & 41.77 & 42.14 \\
Adv. & 63 & 32 & 7.88 & 41.49 & 41.79 & 41.62 & 40.56 & 41.31 & 41.56 \\
Adv. & 127 & 8 & 40.08 & 108.04 & 110.26 & 110.12 & 107.42 & 109.37 & 113.30 \\
Adv. & 127 & 32 & 40.82 & 107.14 & 108.91 & 107.54 & 105.90 & 108.26 & 111.18 \\
\bottomrule
\end{tabular}}
\caption{Variable-matrix batched inference: milliseconds per accepted query. Two independent batch groups per PDE/grid/size, source 0 and seed 11, three repeated measurements. Batch 8 is the first-eight subset of batch 32; the sizes are paired, not independent cohorts. All 10080 query results pass. Two groups provide only a narrow descriptive throughput check.}
\label{tab:deploy-batch_all_grids}
\end{table}

\begin{table}[t]
\centering\small\setlength{\tabcolsep}{3pt}
\resizebox{\linewidth}{!}{%
\begin{tabular}{llllll}
\toprule
PDE & Control & ms saved & 95\% interval & Control/learned & Groups \\
\midrule
Diff. & zero & 0.182 & [0.055, 0.310] & 1.0063 & 2 \\
Diff. & coarse & 4.459 & [4.310, 4.609] & 1.1557 & 2 \\
Diff. & direct & 11.958 & [11.652, 12.263] & 1.4176 & 2 \\
Adv. & zero & 1.242 & [0.922, 1.563] & 1.0118 & 2 \\
Adv. & coarse & 3.006 & [2.861, 3.150] & 1.0284 & 2 \\
Adv. & direct & -65.078 & [-65.237, -64.918] & 0.3853 & 2 \\
\bottomrule
\end{tabular}}
\caption{Batch-32, grid-127, source-0/seed-11 physics-FNO contrasts. Time saved is control minus learned time, in milliseconds per query. Speed ratios are paired geometric means. The bootstrap uses only two batch groups: these are descriptive intervals, not reliable population-level significance evidence. Every listed method solves every query.}
\label{tab:deploy-batch_positive_controls}
\end{table}

\begin{table}[t]
\centering\small\setlength{\tabcolsep}{3pt}
\resizebox{\linewidth}{!}{%
\begin{tabular}{llllllllllll}
\toprule
PDE & $n$ & $Q$ & Status & Direct & Zero & Coarse & Alt. & FNO-D & FNO-P & CNN-D & CNN-P \\
\midrule
Diff. & 31 & 1 & projection & 1.89 & 5.68 & 6.15 & 10.52 & 7.51 & 7.45 & 6.33 & 6.28 \\
Diff. & 31 & 8 & observed & 0.33 & 5.30 & 5.22 & 3.88 & 7.14 & 7.07 & 5.96 & 5.90 \\
Diff. & 31 & 10 & projection & 0.29 & 5.29 & 5.19 & 3.69 & 7.13 & 7.06 & 5.95 & 5.89 \\
Diff. & 31 & 100 & projection & 0.13 & 5.25 & 5.10 & 3.01 & 7.08 & 7.02 & 5.91 & 5.85 \\
Diff. & 31 & 1000 & projection & 0.11 & 5.25 & 5.09 & 2.94 & 7.08 & 7.02 & 5.90 & 5.85 \\
Diff. & 63 & 1 & projection & 8.49 & 10.42 & 11.57 & 17.70 & 12.34 & 12.17 & 11.41 & 11.38 \\
Diff. & 63 & 8 & observed & 1.40 & 9.82 & 9.41 & 7.68 & 11.74 & 11.57 & 10.82 & 10.79 \\
Diff. & 63 & 10 & projection & 1.19 & 9.81 & 9.35 & 7.39 & 11.73 & 11.55 & 10.80 & 10.77 \\
Diff. & 63 & 100 & projection & 0.47 & 9.75 & 9.13 & 6.36 & 11.67 & 11.49 & 10.74 & 10.71 \\
Diff. & 63 & 1000 & projection & 0.39 & 9.74 & 9.10 & 6.26 & 11.66 & 11.49 & 10.74 & 10.70 \\
Diff. & 127 & 1 & projection & 40.25 & 29.26 & 34.29 & 45.95 & 31.04 & 30.68 & 31.29 & 31.23 \\
Diff. & 127 & 8 & observed & 6.46 & 27.75 & 25.92 & 23.51 & 29.59 & 29.23 & 29.84 & 29.76 \\
Diff. & 127 & 10 & projection & 5.50 & 27.70 & 25.68 & 22.87 & 29.55 & 29.19 & 29.80 & 29.72 \\
Diff. & 127 & 100 & projection & 2.02 & 27.55 & 24.82 & 20.57 & 29.40 & 29.04 & 29.65 & 29.57 \\
Diff. & 127 & 1000 & projection & 1.67 & 27.53 & 24.73 & 20.34 & 29.39 & 29.03 & 29.64 & 29.56 \\
Adv. & 31 & 1 & projection & 2.42 & 23.91 & 23.93 & 3.53 & 25.30 & 25.02 & 23.65 & 23.32 \\
Adv. & 31 & 8 & observed & 0.40 & 23.07 & 22.16 & 1.03 & 24.46 & 24.14 & 22.81 & 22.47 \\
Adv. & 31 & 10 & projection & 0.34 & 23.04 & 22.10 & 0.96 & 24.44 & 24.11 & 22.78 & 22.45 \\
Adv. & 31 & 100 & projection & 0.13 & 22.96 & 21.92 & 0.70 & 24.35 & 24.03 & 22.69 & 22.36 \\
Adv. & 31 & 1000 & projection & 0.11 & 22.95 & 21.90 & 0.67 & 24.34 & 24.02 & 22.68 & 22.36 \\
Adv. & 63 & 1 & projection & 8.23 & 46.56 & 46.47 & 26.74 & 47.84 & 47.35 & 46.32 & 46.63 \\
Adv. & 63 & 8 & observed & 1.36 & 45.39 & 43.33 & 17.44 & 46.66 & 46.16 & 45.12 & 45.45 \\
Adv. & 63 & 10 & projection & 1.17 & 45.35 & 43.24 & 17.18 & 46.62 & 46.12 & 45.09 & 45.42 \\
Adv. & 63 & 100 & projection & 0.46 & 45.23 & 42.92 & 16.22 & 46.50 & 46.00 & 44.97 & 45.30 \\
Adv. & 63 & 1000 & projection & 0.39 & 45.22 & 42.89 & 16.13 & 46.49 & 45.99 & 44.95 & 45.28 \\
Adv. & 127 & 1 & projection & 40.66 & 125.13 & 123.71 & 418.75 & 126.51 & 123.74 & 125.65 & 126.92 \\
Adv. & 127 & 8 & observed & 6.54 & 122.56 & 114.03 & 379.42 & 123.92 & 121.14 & 123.10 & 124.37 \\
Adv. & 127 & 10 & projection & 5.56 & 122.48 & 113.76 & 378.30 & 123.85 & 121.07 & 123.03 & 124.29 \\
Adv. & 127 & 100 & projection & 2.04 & 122.22 & 112.76 & 374.29 & 123.58 & 120.81 & 122.77 & 124.02 \\
Adv. & 127 & 1000 & projection & 1.69 & 122.19 & 112.66 & 373.89 & 123.55 & 120.78 & 122.74 & 123.99 \\
\bottomrule
\end{tabular}}
\caption{Fixed-matrix costs in milliseconds per query. Each of four matrix tasks has eight measured right-hand sides. Setup is counted once per matrix and repeats are summarized before averaging matrices. Q=8 is the observed workload; all other Q values are projections of setup/Q plus the measured mean query cost. These are inference-side projections and exclude training cost.}
\label{tab:deploy-cached_all_grids}
\end{table}

\begin{table}[t]
\centering\small\setlength{\tabcolsep}{3pt}
\resizebox{\linewidth}{!}{%
\begin{tabular}{lllllll}
\toprule
PDE & Method & ms/query & Continuation & Probe & Total updates & Accepted \\
\midrule
Diff. & direct & 8.80 & 0.00 & 0.00 & 0.00 & 32/32 \\
Diff. & zero & 10.22 & 60.38 & 0.00 & 60.38 & 32/32 \\
Diff. & coarse & 11.26 & 54.53 & 0.00 & 54.53 & 32/32 \\
Diff. & fno\_data & 12.06 & 59.19 & 0.00 & 59.19 & 32/32 \\
Diff. & min\_residual & 14.25 & 60.38 & 0.00 & 60.38 & 32/32 \\
Diff. & probe5 & 15.30 & 49.59 & 15.00 & 64.59 & 32/32 \\
Adv. & direct & 9.11 & 0.00 & 0.00 & 0.00 & 32/32 \\
Adv. & zero & 39.84 & 91.75 & 0.00 & 91.75 & 32/32 \\
Adv. & coarse & 40.34 & 88.25 & 0.00 & 88.25 & 32/32 \\
Adv. & fno\_data & 41.89 & 92.97 & 0.00 & 92.97 & 32/32 \\
Adv. & min\_residual & 44.43 & 91.75 & 0.00 & 91.75 & 32/32 \\
Adv. & probe5 & 45.54 & 86.34 & 15.00 & 101.34 & 32/32 \\
\bottomrule
\end{tabular}}
\caption{Complete selector costs on 32 independent tasks per PDE, grid 63. Probe work includes discarded candidate updates; continuation alone is not total iterative work. Direct-solve update counts are bookkeeping zeros and are not comparable Krylov complexity. All candidate construction, inference, probes, continuation, and residual checks are timed.}
\label{tab:deploy-selectors}
\end{table}

\begin{table}[t]
\centering\small\setlength{\tabcolsep}{3pt}
\resizebox{\linewidth}{!}{%
\begin{tabular}{lllllllll}
\toprule
PDE & Tolerance & Direct & Zero & Coarse & FNO-D & FNO-P & CNN-D & CNN-P \\
\midrule
Diff. & 1e-04 & 8/8 & 3/8 & 4/8 & 3/8 & 3/8 & 3/8 & 3/8 \\
Diff. & 1e-06 & 8/8 & 2/8 & 2/8 & 2/8 & 2/8 & 2/8 & 2/8 \\
Diff. & 1e-08 & 8/8 & 0/8 & 1/8 & 0/8 & 0/8 & 0/8 & 0/8 \\
Adv. & 1e-04 & 8/8 & 8/8 & 8/8 & 8/8 & 8/8 & 8/8 & 8/8 \\
Adv. & 1e-06 & 8/8 & 6/8 & 6/8 & 6/8 & 6/8 & 6/8 & 6/8 \\
Adv. & 1e-08 & 8/8 & 5/8 & 5/8 & 5/8 & 5/8 & 5/8 & 5/8 \\
\bottomrule
\end{tabular}}
\caption{Full success denominators under the coefficient-log scale-three intervention, grid 63. All other tested conditions accept all endpoints. One numerical repeat is run per task/configuration; the same eight base fields are shared across conditions. Iterative failures are retained and no successful-subset speed is reported.}
\label{tab:deploy-contrast3_acceptance}
\end{table}

\begin{table}[t]
\centering\small\setlength{\tabcolsep}{3pt}
\resizebox{\linewidth}{!}{%
\begin{tabular}{lllllll}
\toprule
PDE & $n$ & State & Resident & Reload & Increment & 95\% interval \\
\midrule
Diff. & 31 & fno\_data & 6.66 & 12.31 & 5.65 & [5.60, 5.69] \\
Diff. & 31 & fno\_physics & 6.59 & 12.25 & 5.66 & [5.61, 5.72] \\
Diff. & 31 & cnn\_data & 5.46 & 10.33 & 4.87 & [4.79, 4.98] \\
Diff. & 31 & cnn\_physics & 5.43 & 10.26 & 4.83 & [4.78, 4.91] \\
Diff. & 63 & fno\_data & 10.83 & 16.44 & 5.62 & [5.55, 5.70] \\
Diff. & 63 & fno\_physics & 10.66 & 16.33 & 5.66 & [5.59, 5.76] \\
Diff. & 63 & cnn\_data & 9.78 & 14.55 & 4.77 & [4.73, 4.82] \\
Diff. & 63 & cnn\_physics & 9.85 & 14.53 & 4.67 & [4.59, 4.73] \\
Diff. & 127 & fno\_data & 26.53 & 32.07 & 5.54 & [5.47, 5.59] \\
Diff. & 127 & fno\_physics & 26.27 & 31.87 & 5.59 & [5.55, 5.64] \\
Diff. & 127 & cnn\_data & 26.47 & 31.23 & 4.76 & [4.74, 4.78] \\
Diff. & 127 & cnn\_physics & 26.49 & 31.26 & 4.76 & [4.73, 4.79] \\
Adv. & 31 & fno\_data & 21.65 & 27.40 & 5.75 & [5.69, 5.84] \\
Adv. & 31 & fno\_physics & 21.65 & 27.37 & 5.72 & [5.63, 5.85] \\
Adv. & 31 & cnn\_data & 20.14 & 24.98 & 4.85 & [4.80, 4.92] \\
Adv. & 31 & cnn\_physics & 19.67 & 24.58 & 4.91 & [4.79, 5.02] \\
Adv. & 63 & fno\_data & 40.49 & 46.25 & 5.76 & [5.63, 5.88] \\
Adv. & 63 & fno\_physics & 39.86 & 45.56 & 5.70 & [5.53, 5.86] \\
Adv. & 63 & cnn\_data & 38.78 & 43.67 & 4.89 & [4.77, 5.02] \\
Adv. & 63 & cnn\_physics & 39.65 & 44.42 & 4.77 & [4.64, 4.90] \\
Adv. & 127 & fno\_data & 106.10 & 111.72 & 5.62 & [5.39, 5.76] \\
Adv. & 127 & fno\_physics & 102.86 & 108.48 & 5.62 & [5.53, 5.70] \\
Adv. & 127 & cnn\_data & 102.87 & 107.98 & 5.11 & [5.04, 5.18] \\
Adv. & 127 & cnn\_physics & 105.16 & 109.90 & 4.75 & [4.53, 5.02] \\
\bottomrule
\end{tabular}}
\caption{Within-process weight reloading versus already resident weights. Milliseconds per query and paired reload increment over four physical tasks per PDE/grid. This includes deserialization/device transfer in an existing Python/CUDA process; it is not a fresh-process cold-start result. Fresh-process results are reported separately in Table~\ref{tab:deploy-fresh_process}.}
\label{tab:deploy-weight_reload}
\end{table}

\FloatBarrier
\subsection{Fresh-process startup contract}
We execute 240 solves in 240 distinct child processes, covering four joint
tasks per PDE, grids 31/127, five paths, and three sequential randomized
repetitions. Neural paths use the source-selected source-0/seed-11 data FNO
or CNN at 6000 updates followed by Poisson correction; direct, zero, and coarse
controls import no Torch. All 96 neural children have no initialized CUDA
context before loading and an initialized context after inference; all 144
classical children and the parent remain Torch-free.

All 240 endpoints satisfy the $10^{-8}$ true-relative-residual contract.
Parent-side independent reference checking gives maximum relative discrete
error $4.11\times10^{-10}$. Eight untimed bridge cases verify identical fields
and matrices and aligned SciPy interpolation against the original Torch
implementation (maximum absolute discrepancy $5.55\times10^{-17}$).
Four checkpoint hashes match the original warm deployment, and the source,
runner, selection, completion and per-record identity gates all pass.

The parent clock runs from child launch until receipt of a flushed residual
acceptance notification. It includes dependency imports, field generation,
assembly, neural weight loading and fresh CUDA initialization, prediction,
correction and notification. Output-vector serialization, process exit, and
parent reference/error checks happen after this endpoint. The OS page cache,
filesystem cache and driver state are not flushed. These measurements ran on
another allocation and use different tasks and endpoints from the warm and
within-process-reload suites; their difference does not identify a CUDA or
weight-loading component cost.

We first take three-repeat medians within each task, then average the four
tasks separately by PDE/grid/path. Task-bootstrap intervals are descriptive;
the small denominator cannot establish population tail latency. Across all
raw repetitions, classical paths range from 0.403 to 0.612 seconds and neural
paths from 2.290 to 2.717 seconds. The table includes all measured paths.

\begin{table}[t]
\centering\small\setlength{\tabcolsep}{3pt}
\resizebox{\linewidth}{!}{%
\begin{tabular}{lllllll}
\toprule
PDE & $n$ & State & Mean(s) & 95\% interval & Task-median range & Accepted \\
\midrule
Diff. & 31 & direct & 0.4133 & [0.4107, 0.4159] & [0.4102, 0.4160] & 12/12 \\
Diff. & 31 & zero & 0.4183 & [0.4170, 0.4196] & [0.4165, 0.4202] & 12/12 \\
Diff. & 31 & coarse & 0.4340 & [0.4163, 0.4518] & [0.4161, 0.4591] & 12/12 \\
Diff. & 31 & fno\_data & 2.3341 & [2.3177, 2.3576] & [2.3142, 2.3698] & 12/12 \\
Diff. & 31 & cnn\_data & 2.3094 & [2.3014, 2.3174] & [2.3005, 2.3197] & 12/12 \\
Diff. & 127 & direct & 0.4964 & [0.4933, 0.4991] & [0.4919, 0.4999] & 12/12 \\
Diff. & 127 & zero & 0.4872 & [0.4846, 0.4898] & [0.4832, 0.4905] & 12/12 \\
Diff. & 127 & coarse & 0.4929 & [0.4891, 0.4957] & [0.4870, 0.4962] & 12/12 \\
Diff. & 127 & fno\_data & 2.4071 & [2.3910, 2.4232] & [2.3872, 2.4276] & 12/12 \\
Diff. & 127 & cnn\_data & 2.3638 & [2.3539, 2.3728] & [2.3503, 2.3757] & 12/12 \\
Adv. & 31 & direct & 0.4146 & [0.4138, 0.4159] & [0.4136, 0.4166] & 12/12 \\
Adv. & 31 & zero & 0.4341 & [0.4317, 0.4358] & [0.4306, 0.4363] & 12/12 \\
Adv. & 31 & coarse & 0.4324 & [0.4292, 0.4358] & [0.4290, 0.4380] & 12/12 \\
Adv. & 31 & fno\_data & 2.3401 & [2.3302, 2.3498] & [2.3263, 2.3534] & 12/12 \\
Adv. & 31 & cnn\_data & 2.3182 & [2.3115, 2.3258] & [2.3086, 2.3296] & 12/12 \\
Adv. & 127 & direct & 0.4983 & [0.4957, 0.5009] & [0.4941, 0.5012] & 12/12 \\
Adv. & 127 & zero & 0.5685 & [0.5583, 0.5850] & [0.5567, 0.5929] & 12/12 \\
Adv. & 127 & coarse & 0.5700 & [0.5585, 0.5811] & [0.5531, 0.5860] & 12/12 \\
Adv. & 127 & fno\_data & 2.4774 & [2.4583, 2.5055] & [2.4546, 2.5199] & 12/12 \\
Adv. & 127 & cnn\_data & 2.4733 & [2.4446, 2.5020] & [2.4423, 2.5146] & 12/12 \\
\bottomrule
\end{tabular}}
\caption{Fresh Python-process latency in seconds, from subprocess launch until the parent receives the true-residual acceptance notification. Each of four physical tasks per PDE/grid is repeated in three distinct processes; repetitions are reduced to task medians before aggregation. FNO/CNN are source-0/seed-11 data models followed by Poisson correction. The task range and descriptive task-bootstrap interval are shown; n=4 limits population inference. Costs include dependency imports, field generation, weight loading and new CUDA contexts for neural paths. Output-vector serialization, child shutdown, and the parent reference solve occur after the endpoint. OS and driver caches are not flushed. All 240 paths accept.}
\label{tab:deploy-fresh_process}
\end{table}

\FloatBarrier
\subsection{Independent confirmation of the batch-cost positive case}
\label{app:batch-confirmation}
The original diffusion, $n=127$, batch-size-32 observation contained only two
batch groups. We therefore froze a conditional, post-hoc confirmation before
running eight new groups (indices 2--9), each containing 32 new tasks.
The original groups 0--1 are excluded from the new estimates. Only the four
loop domains (equation, grid, batch size, group indices) were changed through
an AST wrapper around the frozen deployment implementation; reverse-AST
comparison and source/checkpoint hashes preserve the computational body.
The source dataset and optimization seed remain fixed at 0 and 11; the four
neural paths use their original frozen CNN/FNO data/physics checkpoints.
All seven paths receive one warm-up per group and three measured repetitions
in randomized blocks. Model loading, task generation, and independent direct
reference solutions are outside the online timer. The timer includes matrix
assembly, right-hand-side preparation, initialization (including neural input
transfer and prediction where applicable), preconditioner construction,
numerical correction, synchronization, and true-residual acceptance. Final
reference-error computation is outside the timer. No matrix or factorization
cache is introduced. Thus the result concerns this warm-model online contract,
not startup or full training-and-deployment cost.

The run used node \texttt{Host-D}, one NVIDIA A100-SXM4-80GB, and job
[withheld for anonymous review], with Torch 2.10.0+cu128, NumPy 2.2.6, Python 3.10.19, and one CPU
thread for each recorded BLAS/OpenMP setting. The completed manifest contains
8 units, 168 batch records, and 5,376 query receipts: 256 new physical tasks
measured across seven paths and three repetitions. All receipts passed the
$10^{-8}$ true-relative-residual acceptance rule, with no failures or missing
units. The largest recorded relative residual was $9.99955\times10^{-9}$;
the largest relative error against the independent discrete reference was
$2.84541\times10^{-11}$. These are discrete-system checks.

For each path, we take the median of its three repeated times within each
group, then average the eight group medians. Confidence intervals use 4,000
bootstrap resamples of groups (seed 20260924); paired differences resample
matched groups and are not computed by subtracting marginal intervals.
All predeclared repetitions and groups must pass for a formal timing estimate.
The primary comparison is physics-trained FNO minus zero; other comparisons
are descriptive and receive no multiplicity adjustment.

\begin{table}[t]
\centering\small
\caption{New batch confirmation only. Complete online milliseconds per query;
parentheses show 95\% group-bootstrap intervals. Differences are path minus
zero, so negative is faster. All seven paths have 8/8 accepted groups.}
\label{tab:batch-confirmation}
\begin{tabular}{lrr}
\toprule
Path & Mean cost (95\% CI) & Difference from zero (95\% CI) \\
\midrule
Direct & 40.139 (39.593, 40.513) & 12.670 (12.040, 13.331) \\
Zero + PCG & 27.469 (26.953, 28.020) & --- \\
Coarse + PCG & 31.892 (31.419, 32.397) & 4.423 (4.326, 4.506) \\
CNN (data) + PCG & 29.041 (28.535, 29.575) & 1.572 (1.541, 1.606) \\
CNN (physics) + PCG & 28.984 (28.480, 29.516) & 1.515 (1.474, 1.563) \\
FNO (data) + PCG & 27.746 (27.253, 28.271) & 0.278 (0.241, 0.316) \\
FNO (physics) + PCG & 27.267 (26.765, 27.802) & -0.201 (-0.230, -0.171) \\
\bottomrule
\end{tabular}
\end{table}

The primary paired difference is $-0.201416$\,ms/query, with 95\% CI
$[-0.229989,-0.170747]$; the physics-trained FNO path is faster in all eight
groups. The ratio of mean costs corresponds to a 0.7333\% reduction.
Physics-trained FNO is also cheaper than the implemented coarse and direct
paths by 4.624 and 12.872\,ms/query, respectively (paired 95\% intervals
$[-4.702,-4.544]$ and $[-13.528,-12.246]$). Data-trained FNO is 0.278\,ms/query
slower than zero (interval $[0.241,0.316]$). This contrast supports a positive
case under the specified batching contract, while retaining its small effect
size and dependence on one configuration, one source/optimization seed,
and one node. The condition was selected after the original two-group result;
new tasks provide independent confirmation conditional on that choice,
not a universal speed claim or a correction for the earlier search.

\FloatBarrier
\section{Nested source-information control}
\label{app:source-information}
This extension fixes the source realization while changing its nested training
prefix from 128 to 384 to 1024 examples. A common 64-example source-validation set,
three optimizer seeds, FNO-base architecture, fixed data/residual objectives,
and 6000-update schedule are retained. Twenty-four new training trajectories
are combined with 12 exactly reused 384-example states. The latter pass original
configuration/hash checks and a source-validation recomputation with zero
recorded discrepancy. All 24 new training logs reproduce the frozen schedule
and validation-minimum selection; selected steps range from 5500 to 6000.
Evaluation uses 64 new physical tasks (ID 16 and joint 16 per PDE), grid 31 and
three correctors, producing 3840 accepted endpoints with maximum relative
discrete error $9.493\times10^{-9}$.
The source data arrays remain on the cluster; the local audit checks exact
nesting/hash receipts, complete training logs, and every evaluation identity.
Repeated corrector copies do not create independent prediction evidence.
Intervals below are descriptive and not multiplicity-adjusted. Fixed updates
imply different sample exposures, so the result is conditional on this source
pool and training schedule, not a universal data-scaling law or timing claim.

\begin{table}[t]
\centering\small\setlength{\tabcolsep}{3pt}
\resizebox{\linewidth}{!}{%
\begin{tabular}{lllllll}
\toprule
PDE & Stratum & Loss & $N_s=128$ & 384 & 1024 & 128--1024 reduction [95\% CI] \\
\midrule
Diff. & id & D & 0.0227 & 0.0144 & 0.0125 & 0.0102 [0.0082, 0.0125] \\
Diff. & id & P & 0.0205 & 0.0129 & 0.0112 & 0.0092 [0.0073, 0.0115] \\
Diff. & joint & D & 0.1849 & 0.1416 & 0.1262 & 0.0587 [0.0519, 0.0653] \\
Diff. & joint & P & 0.1803 & 0.1348 & 0.1239 & 0.0564 [0.0454, 0.0683] \\
Adv. & id & D & 0.0203 & 0.0156 & 0.0146 & 0.0057 [0.0046, 0.0072] \\
Adv. & id & P & 0.0192 & 0.0139 & 0.0131 & 0.0061 [0.0049, 0.0076] \\
Adv. & joint & D & 0.1868 & 0.1403 & 0.1316 & 0.0552 [0.0434, 0.0688] \\
Adv. & joint & P & 0.1813 & 0.1374 & 0.1279 & 0.0534 [0.0408, 0.0684] \\
\bottomrule
\end{tabular}}
\caption{Complete prediction configurations, initial relative discrete $L^2$ error. D is data-only; P adds the fixed residual weight $\lambda=0.001$. Both PDEs use one nested source realization, common validation 64, FNO base, and a common 6000-update schedule. Within each of 16 physical test tasks per PDE/stratum, optimizer seeds 11/23/37 are averaged before summaries and descriptive 4000-draw bootstrap intervals. Correction copies do not create extra prediction evidence.}
\label{tab:source-prediction}
\end{table}

\begin{table}[t]
\centering\small\setlength{\tabcolsep}{3pt}
\resizebox{\linewidth}{!}{%
\begin{tabular}{llllll}
\toprule
PDE & Stratum & Loss & 128--384 & 384--1024 & 128--1024 \\
\midrule
Diff. & id & D & 0.0083 [0.0068, 0.0100] & 0.0019 [0.0010, 0.0029] & 0.0102 [0.0082, 0.0125] \\
Diff. & id & P & 0.0075 [0.0062, 0.0090] & 0.0017 [0.0007, 0.0028] & 0.0092 [0.0073, 0.0115] \\
Diff. & joint & D & 0.0433 [0.0345, 0.0518] & 0.0154 [0.0091, 0.0221] & 0.0587 [0.0519, 0.0653] \\
Diff. & joint & P & 0.0456 [0.0340, 0.0591] & 0.0109 [0.0054, 0.0160] & 0.0564 [0.0454, 0.0683] \\
Adv. & id & D & 0.0048 [0.0036, 0.0064] & 0.0010 [0.0005, 0.0016] & 0.0057 [0.0046, 0.0072] \\
Adv. & id & P & 0.0053 [0.0041, 0.0068] & 0.0008 [0.0003, 0.0014] & 0.0061 [0.0049, 0.0076] \\
Adv. & joint & D & 0.0465 [0.0376, 0.0560] & 0.0087 [0.0042, 0.0138] & 0.0552 [0.0434, 0.0688] \\
Adv. & joint & P & 0.0439 [0.0344, 0.0552] & 0.0096 [0.0055, 0.0137] & 0.0534 [0.0408, 0.0684] \\
\bottomrule
\end{tabular}}
\caption{All 24 predeclared paired source-count prediction contrasts. Positive values mean lower prediction error with more source examples, holding validation information and optimizer-update schedule fixed. These intervals are conditional on the single source pool and fitted states; they are not multiplicity-adjusted or population-level data-scaling guarantees.}
\label{tab:source-prediction_pairs}
\end{table}

\begin{table}[t]
\centering\small\setlength{\tabcolsep}{3pt}
\resizebox{\linewidth}{!}{%
\begin{tabular}{llllllll}
\toprule
Stratum & Loss & Corrector & Zero & Coarse & 128 & 384 & 1024 \\
\midrule
id & D & CG & 152.88 & 144.56 & 149.06 & 147.10 & 146.58 \\
id & D & Poisson-PCG & 18.56 & 16.62 & 16.62 & 16.46 & 16.38 \\
id & D & AMG-PCG & 7.88 & 6.12 & 6.92 & 6.73 & 6.73 \\
id & P & CG & 152.88 & 144.56 & 148.46 & 146.40 & 145.44 \\
id & P & Poisson-PCG & 18.56 & 16.62 & 16.40 & 15.98 & 15.83 \\
id & P & AMG-PCG & 7.88 & 6.12 & 6.81 & 6.54 & 6.44 \\
joint & D & CG & 266.19 & 257.69 & 276.25 & 273.27 & 272.00 \\
joint & D & Poisson-PCG & 54.44 & 50.62 & 53.54 & 52.96 & 52.67 \\
joint & D & AMG-PCG & 7.94 & 7.00 & 7.65 & 7.60 & 7.54 \\
joint & P & CG & 266.19 & 257.69 & 274.98 & 272.62 & 272.02 \\
joint & P & Poisson-PCG & 54.44 & 50.62 & 53.06 & 52.56 & 52.31 \\
joint & P & AMG-PCG & 7.94 & 7.00 & 7.67 & 7.54 & 7.46 \\
\bottomrule
\end{tabular}}
\caption{Diffusion complete mean correction counts at grid 31 and true relative residual $10^{-8}$. Each neural cell averages three seeds within 16 physical tasks; each classical control has 16 endpoints and is repeated in D/P rows only for readability. All 3840 endpoints across both PDEs succeed. Controls are independent of source sample count. Counts are not wall-clock utility.}
\label{tab:source-diffusion_updates}
\end{table}

\begin{table}[t]
\centering\small\setlength{\tabcolsep}{3pt}
\resizebox{\linewidth}{!}{%
\begin{tabular}{llllll}
\toprule
Stratum & Loss & Corrector & 128--384 & 384--1024 & 128--1024 \\
\midrule
id & D & CG & 1.96 [0.96, 2.79] & 0.52 [-0.04, 1.04] & 2.48 [1.40, 3.44] \\
id & D & Poisson-PCG & 0.17 [0.06, 0.29] & 0.08 [-0.02, 0.19] & 0.25 [0.10, 0.42] \\
id & D & AMG-PCG & 0.19 [-0.02, 0.40] & 0.00 [-0.08, 0.08] & 0.19 [0.00, 0.38] \\
id & P & CG & 2.06 [1.37, 2.79] & 0.96 [0.25, 1.73] & 3.02 [2.10, 4.06] \\
id & P & Poisson-PCG & 0.42 [0.27, 0.54] & 0.15 [0.06, 0.25] & 0.56 [0.42, 0.69] \\
id & P & AMG-PCG & 0.27 [0.12, 0.44] & 0.10 [0.00, 0.23] & 0.38 [0.23, 0.52] \\
joint & D & CG & 2.98 [1.71, 4.21] & 1.27 [0.46, 2.10] & 4.25 [3.02, 5.50] \\
joint & D & Poisson-PCG & 0.58 [0.46, 0.71] & 0.29 [0.15, 0.44] & 0.88 [0.73, 1.00] \\
joint & D & AMG-PCG & 0.04 [-0.06, 0.17] & 0.06 [0.00, 0.17] & 0.10 [-0.02, 0.25] \\
joint & P & CG & 2.35 [1.54, 3.23] & 0.60 [-0.25, 1.50] & 2.96 [1.92, 4.12] \\
joint & P & Poisson-PCG & 0.50 [0.35, 0.65] & 0.25 [0.10, 0.38] & 0.75 [0.60, 0.90] \\
joint & P & AMG-PCG & 0.13 [0.00, 0.25] & 0.08 [0.02, 0.17] & 0.21 [0.08, 0.33] \\
\bottomrule
\end{tabular}}
\caption{Diffusion all paired changes in correction work with source count. Values are updates saved by the larger count, with descriptive task-bootstrap 95\% intervals. Negative or near-zero changes are retained. The same 16 tasks are paired across counts and seeds are averaged within task; the interval does not resample the source dataset.}
\label{tab:source-diffusion_paired_updates}
\end{table}

\begin{table}[t]
\centering\small\setlength{\tabcolsep}{3pt}
\resizebox{\linewidth}{!}{%
\begin{tabular}{llllllll}
\toprule
Stratum & Loss & Corrector & Zero & Coarse & 128 & 384 & 1024 \\
\midrule
id & D & GMRES & 218.81 & 199.62 & 164.48 & 166.96 & 159.48 \\
id & D & Poisson-GMRES & 40.62 & 40.31 & 39.33 & 39.21 & 38.92 \\
id & D & ILU-GMRES & 3.00 & 3.00 & 3.00 & 3.00 & 3.00 \\
id & P & GMRES & 218.81 & 199.62 & 165.71 & 164.10 & 158.08 \\
id & P & Poisson-GMRES & 40.62 & 40.31 & 38.60 & 38.71 & 38.52 \\
id & P & ILU-GMRES & 3.00 & 3.00 & 3.00 & 3.00 & 3.00 \\
joint & D & GMRES & 346.75 & 289.44 & 305.40 & 302.19 & 301.38 \\
joint & D & Poisson-GMRES & 75.50 & 73.38 & 75.88 & 75.67 & 75.40 \\
joint & D & ILU-GMRES & 3.00 & 3.00 & 3.00 & 3.00 & 3.00 \\
joint & P & GMRES & 346.75 & 289.44 & 305.98 & 299.75 & 300.00 \\
joint & P & Poisson-GMRES & 75.50 & 73.38 & 75.77 & 75.10 & 74.60 \\
joint & P & ILU-GMRES & 3.00 & 3.00 & 3.00 & 3.00 & 3.00 \\
\bottomrule
\end{tabular}}
\caption{Advection complete mean correction counts at grid 31 and true relative residual $10^{-8}$. Each neural cell averages three seeds within 16 physical tasks; each classical control has 16 endpoints and is repeated in D/P rows only for readability. All 3840 endpoints across both PDEs succeed. Controls are independent of source sample count. Counts are not wall-clock utility.}
\label{tab:source-advection_updates}
\end{table}

\begin{table}[t]
\centering\small\setlength{\tabcolsep}{3pt}
\resizebox{\linewidth}{!}{%
\begin{tabular}{llllll}
\toprule
Stratum & Loss & Corrector & 128--384 & 384--1024 & 128--1024 \\
\midrule
id & D & GMRES & -2.48 [-7.35, 2.44] & 7.48 [2.23, 12.90] & 5.00 [-0.04, 9.92] \\
id & D & Poisson-GMRES & 0.13 [-0.31, 0.60] & 0.29 [-0.02, 0.67] & 0.42 [0.13, 0.79] \\
id & D & ILU-GMRES & 0.00 [0.00, 0.00] & 0.00 [0.00, 0.00] & 0.00 [0.00, 0.00] \\
id & P & GMRES & 1.60 [-2.17, 5.23] & 6.02 [0.65, 11.00] & 7.63 [0.85, 14.44] \\
id & P & Poisson-GMRES & -0.10 [-0.56, 0.27] & 0.19 [-0.10, 0.56] & 0.08 [-0.29, 0.35] \\
id & P & ILU-GMRES & 0.00 [0.00, 0.00] & 0.00 [0.00, 0.00] & 0.00 [0.00, 0.00] \\
joint & D & GMRES & 3.21 [-4.77, 11.77] & 0.81 [-6.42, 8.23] & 4.02 [-5.17, 13.02] \\
joint & D & Poisson-GMRES & 0.21 [-0.46, 0.94] & 0.27 [-0.56, 1.04] & 0.48 [-0.46, 1.40] \\
joint & D & ILU-GMRES & 0.00 [0.00, 0.00] & 0.00 [0.00, 0.00] & 0.00 [0.00, 0.00] \\
joint & P & GMRES & 6.23 [-1.40, 15.04] & -0.25 [-5.92, 4.92] & 5.98 [-3.62, 16.77] \\
joint & P & Poisson-GMRES & 0.67 [-0.02, 1.35] & 0.50 [-0.25, 1.31] & 1.17 [0.37, 2.17] \\
joint & P & ILU-GMRES & 0.00 [0.00, 0.00] & 0.00 [0.00, 0.00] & 0.00 [0.00, 0.00] \\
\bottomrule
\end{tabular}}
\caption{Advection all paired changes in correction work with source count. Values are updates saved by the larger count, with descriptive task-bootstrap 95\% intervals. Negative or near-zero changes are retained. The same 16 tasks are paired across counts and seeds are averaged within task; the interval does not resample the source dataset.}
\label{tab:source-advection_paired_updates}
\end{table}

\begin{table}[t]
\centering\small\setlength{\tabcolsep}{3pt}
\resizebox{\linewidth}{!}{%
\begin{tabular}{lllll}
\toprule
PDE & Loss & Corrector & Saved vs zero [95\% CI] & Saved vs coarse [95\% CI] \\
\midrule
Diff. & D & CG & -5.81 [-8.08, -3.83] & -14.31 [-16.67, -12.10] \\
Diff. & D & Poisson-PCG & 1.77 [1.42, 2.13] & -2.04 [-2.35, -1.77] \\
Diff. & D & AMG-PCG & 0.40 [0.19, 0.62] & -0.54 [-0.77, -0.31] \\
Diff. & P & CG & -5.83 [-7.73, -4.08] & -14.33 [-16.96, -11.92] \\
Diff. & P & Poisson-PCG & 2.12 [1.79, 2.48] & -1.69 [-2.00, -1.40] \\
Diff. & P & AMG-PCG & 0.48 [0.27, 0.69] & -0.46 [-0.67, -0.23] \\
Adv. & D & GMRES & 45.38 [31.75, 59.96] & -11.94 [-21.81, -2.31] \\
Adv. & D & Poisson-GMRES & 0.10 [-1.44, 1.50] & -2.02 [-3.33, -0.69] \\
Adv. & D & ILU-GMRES & 0.00 [0.00, 0.00] & 0.00 [0.00, 0.00] \\
Adv. & P & GMRES & 46.75 [32.50, 61.60] & -10.56 [-19.73, -1.33] \\
Adv. & P & Poisson-GMRES & 0.90 [-0.23, 2.15] & -1.23 [-2.25, -0.08] \\
Adv. & P & ILU-GMRES & 0.00 [0.00, 0.00] & 0.00 [0.00, 0.00] \\
\bottomrule
\end{tabular}}
\caption{Joint-shift largest-source states versus both competent initialization controls. Positive values favor the learned state. Increased source information can improve the predictor while leaving a CG disadvantage or a disadvantage against coarse initialization. Fixed residual training is not a newly selected or retuned loss in this experiment.}
\label{tab:source-joint1024_controls}
\end{table}

\FloatBarrier
\section{Same-trajectory selection and fixed target-information budgets}
\label{app:information}
The target distribution is the frozen joint shift at grid $n=31$.
One fixed pool of 80 target labels per equation is allocated as 16/64,
40/40 or 64/16 training/validation pairs. Labels can change roles across
allocations, but are disjoint within each allocation; these are not independent
label pools. Two old pretrained source states and scratch, three optimization
seeds, two learning rates and both PDEs give 108 actual training trajectories.
The 200/800/1600-update checkpoints retain both the final state and the
validation-best state from the same trajectory, yielding 648 evaluated
checkpoint-policy states plus 12 frozen controls. Each equation uses 128 fresh
test tasks, separate from the adaptation label pool and all model selection.

An independent audit reconstructs every checkpoint and learning-rate selection
from complete validation curves and verifies the pre-test selection hash.
All 108 trajectories, 648 checkpoint states and 324 per-seed rate-selected
states are present with finite results. The latter are selections from the
former, not additional training. Within-rate contrasts identify checkpoint
policy comparisons; allowing each policy to select its own rate gives the
broader strategy comparison. Validation selection is sometimes neutral or
slightly unfavorable, particularly with 64/16 advection labels; these entries
are retained. All trajectories were actually run to 1600 updates, so selecting
an earlier checkpoint does not itself establish realized compute savings.
The source weights predate the 6000-update source expansion; these adaptation
results should not be attributed to the new source models.

\begin{table}[t]
\centering\small\setlength{\tabcolsep}{3pt}
\caption{Same-trajectory checkpoint selection at budget 1600. Entries are fixed-final minus validation-best relative $L^2$ errors, in units of $10^{-3}$; positive values favor validation selection. Intervals are descriptive paired-task 95\% bootstraps with 4000 resamples after averaging three optimization seeds within each of 128 test tasks per equation. No multiplicity adjustment is applied. Both rates are fixed in their respective columns; all other budgets are retained in the accompanying CSVs.}
\label{tab:information-same-trajectory}
\resizebox{\linewidth}{!}{%
\begin{tabular}{lllll}\toprule
Equation & State & Train/val & $\eta=3\times10^{-4}$ & $\eta=10^{-3}$\\\midrule
Diff. & Source 0 & 16/64 & $+30.536\;[+25.942,+35.029]$ & $+25.908\;[+21.878,+29.886]$\\
Diff. & Source 0 & 40/40 & $+9.168\;[+5.938,+12.474]$ & $+9.544\;[+6.319,+12.815]$\\
Diff. & Source 0 & 64/16 & $+2.645\;[+0.932,+4.362]$ & $+2.916\;[+1.121,+4.721]$\\
Diff. & Source 1 & 16/64 & $+32.450\;[+27.984,+37.026]$ & $+31.000\;[+27.586,+34.498]$\\
Diff. & Source 1 & 40/40 & $+13.045\;[+7.635,+18.642]$ & $+15.093\;[+10.768,+19.748]$\\
Diff. & Source 1 & 64/16 & $+4.651\;[+2.237,+7.260]$ & $+5.294\;[+2.897,+7.751]$\\
Diff. & Scratch & 16/64 & $+50.715\;[+41.401,+60.105]$ & $+57.667\;[+48.503,+67.028]$\\
Diff. & Scratch & 40/40 & $+14.456\;[+9.169,+20.034]$ & $+21.659\;[+15.936,+27.552]$\\
Diff. & Scratch & 64/16 & $+0.268\;[+0.152,+0.387]$ & $+1.211\;[+0.413,+2.027]$\\
Adv. & Source 0 & 16/64 & $+13.632\;[+11.724,+15.610]$ & $+19.834\;[+17.657,+22.089]$\\
Adv. & Source 0 & 40/40 & $+0.511\;[+0.230,+0.814]$ & $+4.354\;[+3.639,+5.087]$\\
Adv. & Source 0 & 64/16 & $-0.041\;[-0.231,+0.146]$ & $+0.078\;[-0.559,+0.722]$\\
Adv. & Source 1 & 16/64 & $+15.580\;[+13.461,+17.739]$ & $+20.980\;[+18.614,+23.446]$\\
Adv. & Source 1 & 40/40 & $+0.717\;[+0.363,+1.082]$ & $+3.158\;[+2.403,+3.937]$\\
Adv. & Source 1 & 64/16 & $-0.169\;[-0.406,+0.073]$ & $+0.410\;[-0.438,+1.256]$\\
Adv. & Scratch & 16/64 & $+59.738\;[+52.870,+66.970]$ & $+93.113\;[+84.820,+101.272]$\\
Adv. & Scratch & 40/40 & $-0.084\;[-0.166,-0.004]$ & $+6.132\;[+4.959,+7.331]$\\
Adv. & Scratch & 64/16 & $+0.000\;[+0.000,+0.000]$ & $+0.022\;[-0.471,+0.504]$\\
\bottomrule\end{tabular}}
\end{table}

\begin{table}[t]
\centering\small\setlength{\tabcolsep}{3pt}
\caption{Fixed 80-label allocation, budget 1600. Each checkpoint policy selects its own learning rate on validation data. Errors average three seeds within each of 128 independent tasks per equation; selected validation-best steps list the three seeds. These columns compare complete selection policies, not only checkpoint timing.}
\label{tab:information-label-allocation}
\resizebox{\linewidth}{!}{%
\begin{tabular}{llllll}\toprule
Equation & Initial state & Train/val & Fixed-final & Validation-best & Selected steps\\\midrule
Diff. & Source 0 & 16/64 & 0.1955 & 0.1677 & 50,25,25\\
Diff. & Source 1 & 16/64 & 0.1990 & 0.1665 & 50,25,25\\
Diff. & Scratch & 16/64 & 0.3032 & 0.2525 & 350,375,300\\
Diff. & Source 0 & 40/40 & 0.1657 & 0.1581 & 100,250,75\\
Diff. & Source 1 & 40/40 & 0.1693 & 0.1567 & 50,50,25\\
Diff. & Scratch & 40/40 & 0.2308 & 0.2164 & 525,775,400\\
Diff. & Source 0 & 64/16 & 0.1542 & 0.1512 & 1325,125,1400\\
Diff. & Source 1 & 64/16 & 0.1551 & 0.1504 & 225,50,1575\\
Diff. & Scratch & 64/16 & 0.1896 & 0.1893 & 1600,1600,1525\\
Adv. & Source 0 & 16/64 & 0.1473 & 0.1328 & 75,50,150\\
Adv. & Source 1 & 16/64 & 0.1459 & 0.1297 & 50,50,125\\
Adv. & Scratch & 16/64 & 0.2686 & 0.2095 & 200,175,175\\
Adv. & Source 0 & 40/40 & 0.1066 & 0.1061 & 1400,375,1575\\
Adv. & Source 1 & 40/40 & 0.1054 & 0.1048 & 1525,275,1400\\
Adv. & Scratch & 40/40 & 0.1524 & 0.1517 & 1600,1600,1025\\
Adv. & Source 0 & 64/16 & 0.0949 & 0.0949 & 1325,500,1125\\
Adv. & Source 1 & 64/16 & 0.0936 & 0.0940 & 1250,350,1600\\
Adv. & Scratch & 64/16 & 0.1248 & 0.1249 & 1525,1600,1600\\
\bottomrule\end{tabular}}
\end{table}

\FloatBarrier
\section{Nonlinear training sensitivity and independent confirmation}
\label{app:nonlinear-long}
We retrain six FNO states on the original 512/96 source pool for 6000 updates,
using the original implementation with its cosine schedule stretched to the new
budget. This is a post-hoc sensitivity analysis of the training recipe, not a
prefix continuation or proof of training saturation. The original 112 target
tasks give paired budget comparisons. A separate set of 56 new tasks, eight per
stratum with disjoint identifiers, reuses these six states without retraining.
Both cohorts retain three grids, three tolerances, three seeds per neural
objective, and zero, linearized and coarse controls.

All 9072 long-budget and 4536 confirmation endpoints accept, with respectively
27216 and 13608 accepted timing repetitions. Independent local inspection of
all six selected checkpoint files matches the confirmation provenance hashes.
Original-environment residual-bound diagnostics verify 9072 original,
9072 long-budget, and 4536 confirmation records with no observed floating-point
inequality violations. The earlier 189-record cross-platform reference mismatch
is resolved by this same-environment check; it was not an original solve failure.
These are numerical checks, not interval-arithmetic certificates.

Mean prediction error decreases in all seven strata for both training objectives, but
correction improvements are uneven. Joint-shift gains relative to zero become
positive in the longer recipe and in independent confirmation. The original
extreme-stratum negative effect does not receive a sign-definite confirmation:
its new-task interval includes zero. Coarse initialization still matches or
improves the displayed correction counts. Complete warm cost also retains the
cheap banded classical paths' advantage; timing changes across allocations are
not attributed causally to the training budget.

\begin{table}[t]
\centering\small\setlength{\tabcolsep}{3pt}
\resizebox{\linewidth}{!}{%

\begin{tabular}{lllllll}\toprule
Stratum & Objective & Error 1200 & Error 6000 & Steps 1200 & Steps 6000 & Steps saved [95\% CI]\\\midrule
ID & data & 0.0465 & 0.0146 & 2.771 & 2.167 & +0.604 [+0.417,+0.771]\\
ID & physics & 0.0453 & 0.0137 & 2.771 & 2.188 & +0.583 [+0.396,+0.771]\\
Low diffusion & data & 0.5513 & 0.3130 & 4.958 & 4.458 & +0.500 [+0.271,+0.708]\\
Low diffusion & physics & 0.5494 & 0.3204 & 4.958 & 4.542 & +0.417 [+0.208,+0.625]\\
Strong reaction & data & 0.2398 & 0.1907 & 3.792 & 3.729 & +0.062 [-0.125,+0.250]\\
Strong reaction & physics & 0.2360 & 0.1759 & 3.833 & 3.708 & +0.125 [-0.063,+0.312]\\
Rough & data & 0.5079 & 0.4019 & 3.188 & 3.125 & +0.062 [-0.083,+0.208]\\
Rough & physics & 0.3334 & 0.3069 & 3.021 & 3.000 & +0.021 [-0.063,+0.125]\\
Joint & data & 1.4611 & 0.8857 & 6.667 & 5.958 & +0.708 [+0.313,+1.083]\\
Joint & physics & 1.2247 & 0.7489 & 6.354 & 5.750 & +0.604 [+0.271,+0.958]\\
Linear & data & 0.1346 & 0.0744 & 1.000 & 1.000 & +0.000 [+0.000,+0.000]\\
Linear & physics & 0.1343 & 0.0734 & 1.000 & 1.000 & +0.000 [+0.000,+0.000]\\
Extreme & data & 3.1720 & 1.6885 & 12.729 & 12.188 & +0.542 [-0.229,+1.375]\\
Extreme & physics & 2.8958 & 1.4736 & 12.229 & 11.938 & +0.292 [-0.396,+0.854]\\
\bottomrule\end{tabular}}
\caption{Exploratory common-task source-training sensitivity at $n=63$ and residual tolerance $10^{-8}$. Six states are retrained with a stretched 6000-step cosine schedule; 112 old test tasks provide paired contrasts, with 16 tasks per stratum. Positive update differences favor 6000 steps. This is not a same-training-trajectory causal budget contrast.}
\end{table}

\begin{table}[t]
\centering\small\setlength{\tabcolsep}{3pt}
\resizebox{\linewidth}{!}{%

\begin{tabular}{lllllll}\toprule
Stratum & Objective & Initial error & Learned steps & Zero steps & Coarse steps & Saved vs zero [95\% CI]\\\midrule
ID & data & 0.0103 & 2.083 & 4.625 & 2.000 & +2.542 [+2.208,+2.875]\\
ID & physics & 0.0098 & 2.083 & 4.625 & 2.000 & +2.542 [+2.208,+2.875]\\
Low diffusion & data & 0.2623 & 4.333 & 5.625 & 2.000 & +1.292 [+0.958,+1.625]\\
Low diffusion & physics & 0.2730 & 4.292 & 5.625 & 2.000 & +1.333 [+1.042,+1.625]\\
Strong reaction & data & 0.8148 & 3.833 & 4.625 & 2.000 & +0.792 [+0.292,+1.375]\\
Strong reaction & physics & 0.7060 & 3.708 & 4.625 & 2.000 & +0.917 [+0.417,+1.500]\\
Rough & data & 0.3197 & 3.000 & 3.750 & 2.000 & +0.750 [+0.250,+1.375]\\
Rough & physics & 0.2458 & 2.958 & 3.750 & 2.000 & +0.792 [+0.250,+1.375]\\
Joint & data & 0.8362 & 5.875 & 6.875 & 3.000 & +1.000 [+0.708,+1.292]\\
Joint & physics & 0.7371 & 5.667 & 6.875 & 3.000 & +1.208 [+0.792,+1.583]\\
Linear & data & 0.1172 & 1.000 & 1.000 & 1.000 & +0.000 [+0.000,+0.000]\\
Linear & physics & 0.1188 & 1.000 & 1.000 & 1.000 & +0.000 [+0.000,+0.000]\\
Extreme & data & 1.5952 & 11.792 & 11.125 & 6.125 & -0.667 [-1.668,+0.418]\\
Extreme & physics & 1.4075 & 11.417 & 11.125 & 6.125 & -0.292 [-1.125,+0.792]\\
\bottomrule\end{tabular}}
\caption{Independent confirmation: 56 new physical tasks, eight per stratum, reuse the same six 6000-step models without additional training. $n=63$, true residual tolerance $10^{-8}$. Task bootstrap intervals are conditional on these fitted states and small strata; they are not combined with the original 112 tasks.}
\end{table}

\FloatBarrier
\subsection{Retained initial nonlinear results and visual comparison}
\paragraph{A nonlinear boundary test separates three notions of success.}
We independently implement the monotone semilinear problem
$-\nu u''+\gamma u^3=f$ on $[0,1]$ with zero Dirichlet conditions,
$\nu>0$, and $\gamma\geq0$.
Six FNO states (data/residual-augmented objectives, three seeds) are trained for
1200 updates on one common 512/96 source train/validation pool.
Seven prespecified parameter/forcing strata contain 16 independent tasks each,
paired across $n=63,127,255$.
Damped Newton with exact banded Jacobian solves corrects zero, linearized,
coarse-grid and learned initializations at three true-residual tolerances.
All 9072 correction endpoints and 27216 timed repetitions pass acceptance;
all 336 strict discrete references and independent continuous collocation
crosschecks succeed. Optimization seeds and repeated discretizations do not
increase the 112-task count.

At $n=63$ and tolerance $10^{-8}$, the data-trained FNO has mean initial
relative error $0.0465$ in distribution and needs 2.771 Newton updates,
versus 4.250 from zero: 1.479 updates saved, with a descriptive paired-task
95\% interval $[1.146,1.792]$. Coarse initialization needs only 2.000 updates.
Under joint shift the learned-versus-zero saving is $0.021$
$[-0.333,0.396]$; at the extreme low-diffusion/high-reaction boundary it
becomes $-1.604$ $[-2.313,-0.937]$. The residual-augmented objective improves
high-frequency-forcing prediction error from $0.508$ to $0.333$, but does
not beat the coarse initialization's correction work. At $\gamma=0$,
learned, zero and coarse states all require one exact Newton update,
consistent with the linear limit rather than a failed prediction experiment.

The complete-cost comparison is stricter. For the in-distribution
$n=63$ tasks at tolerance $10^{-8}$, warm-resident median times are
$0.612/0.609/0.970/1.824$\,ms for zero/linearized/coarse/data-FNO paths.
At the extreme boundary they are $1.863/2.983/2.725/2.924$\,ms.
Thus even the coarse method's lower correction count can be offset by
initialization work. These measurements include construction, transfers,
initialization, correction and true-residual verification, with one CPU
numerical thread and one A100 for inference. They exclude offline training
and model loading; initialization is warmed, but numerical correction has
no separate warmup. The inexpensive one-dimensional banded solver is a
declared workload, not a universal hardware comparison.

\paragraph{Longer training and independent confirmation.}
Six states use a longer 6000-update training budget with a stretched cosine schedule; this is recipe sensitivity, not a shared-trajectory prefix comparison. On the original 112 tasks, data-FNO Newton savings against zero at $n=63$, tolerance $10^{-8}$, increase from 1.479 to 2.083 in distribution and 0.021 to 0.729 under joint shift. A separate 56-task confirmation gives 2.542 [2.208,2.875] and 1.000 [0.708,1.292], respectively. Extreme-stratum savings are $-1.063$ [$-1.958$,$-0.104$] on old tasks; confirmation gives $-0.667$ [$-1.668$,0.418], leaving the sign uncertain. Both cohorts pass all 9072/4536 endpoints. Coarse initialization still matches or improves correction counts; ID warm data-FNO/zero costs are 1.723/0.609ms.

\paragraph{Discrete acceptance is distinct from continuum accuracy.}
For $F_h(u)=\nu L_hu+\gamma u^{\odot3}-f$, the standard strong-monotonicity
constant is $m_h=\nu\lambda_{\min}(L_h)$, where
$\lambda_{\min}(L_h)=4h^{-2}\sin^2(\pi h/2)$.
Since $(a-b)(a^3-b^3)\geq0$, Cauchy--Schwarz gives
$\|u-u_h^*\|_2\leq\|F_h(u)\|_2/m_h$.
This elementary discrete bound is parameter dependent and does not control
the discretization error. The strict references have maximum relative
residual $8.92\times10^{-12}$, yet their mean relative differences from
independent adaptive continuous collocation solutions on the extreme
stratum are $4.644\%$, $1.173\%$, and $0.223\%$ at
$n=63,127,255$, respectively. The corresponding maxima are
$7.070\%$, $2.179\%$, and $0.347\%$.
Refinement to $2n+1$ supplies a second numerical crosscheck.
Collocation convergence at an RMS residual tolerance of $10^{-7}$ is
not a rigorous pointwise error certificate. 

\begin{figure}[htbp]
\centering\includegraphics[width=\linewidth]{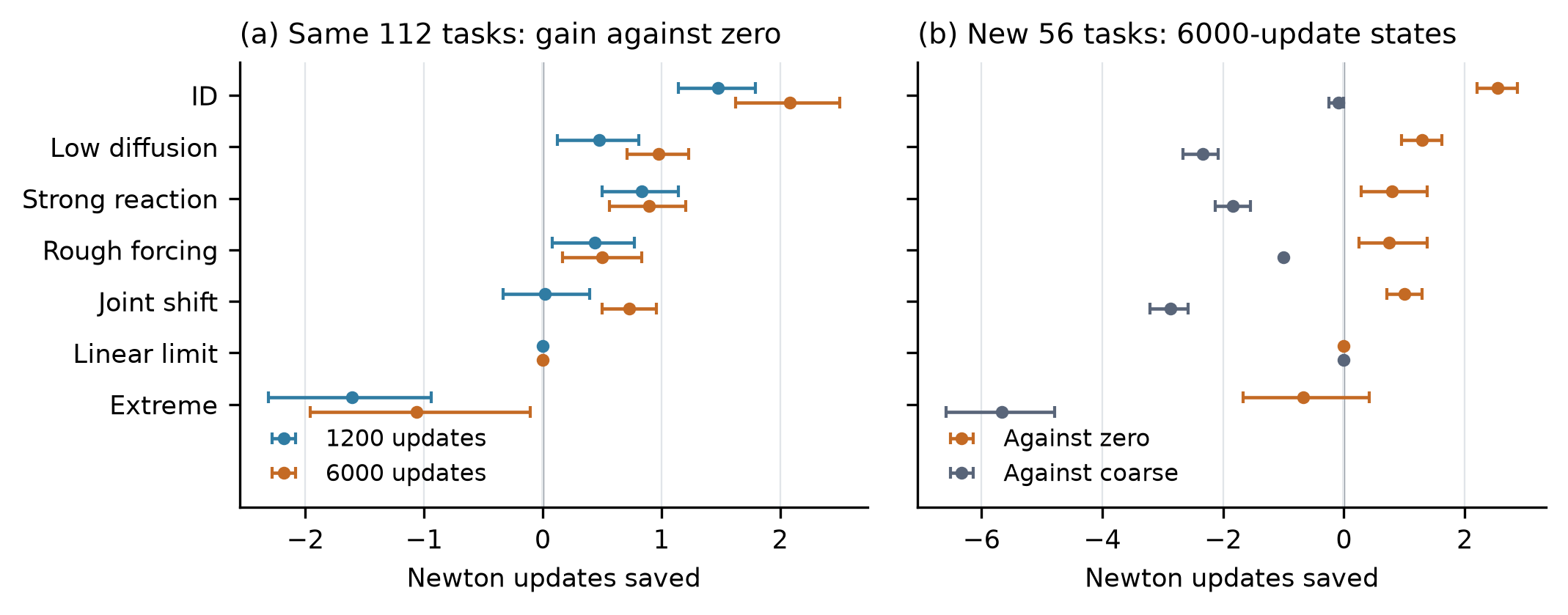}
\caption{Training-recipe sensitivity retained from the source manuscript. Left: the same 112 tasks under two training schedules. Right: 56-task independent confirmation of the longer-budget states, compared with zero and coarse controls. At $n=63$ and residual tolerance $10^{-8}$, the extreme-stratum contrast against zero is not sign-confirmed in the new cohort. Intervals are descriptive paired-task 95\% intervals.}
\end{figure}

\FloatBarrier
\section{Official HINTS mechanism and online timing audit}
\label{app:hints-official}
\paragraph{An official pretrained positive case.}
To distinguish one-time initialization from repeated learned correction, we
also execute the official pretrained HINTS 1D Poisson example
\citep{hints2022}, without retraining or changing its weights, architecture,
Jacobi relaxation ($\omega=2/3$), or 25-update hybrid schedule.
The external measurement wrapper matches the official 400-update residual
trajectory to $2.23\times10^{-16}$ maximum absolute discrepancy.
After freezing an initial 32-case audit, we extend the same protocol to all
750 official test cases as a post-hoc coverage check; the repeated prefix
is not additional independent evidence.
At relative residual $10^{-8}$, HINTS accepts 750/750 cases within 400 updates,
versus 0/750 for Jacobi or a single neural initialization followed by Jacobi.
With a 10,000-update cap all three succeed: mean update counts are
4992.71, 3899.80, and 138.84, respectively.
The last number comprises 133.51 numerical updates and 5.33 neural calls,
which are not equal-cost operations.
Thus initialization itself helps this baseline, and repeated correction helps
substantially more. At $10^{-10}$, Jacobi retains two capped cases (748/750),
while both neural variants accept all cases.
These are mechanism results on 29 interior degrees of freedom, not a
large-scale speed benchmark. Logged trajectories include reference-error
diagnostics; the official setup also computes a reference solution and an
eigendecomposition. Neither those timings nor separately reported direct-solve
checks establish complete deployment cost.
\paragraph{A separate online cost check.}
We time the fixed first 32 cases in three randomized blocks with CPU-resident
inputs and an already loaded model. Every query includes identical matrix
assembly, necessary initialization, and true-residual stopping; reference
solutions, eigendecompositions, and diagnostic errors are absent from the
timed path. Iteration counts exactly match the frozen audit.
Medians across per-task medians are 18.809 ms (Jacobi), 14.339 ms
(one neural initialization), 1.010 ms (HINTS), 0.0356 ms (dense direct), and
0.0299 ms (tridiagonal direct). Thus repeated correction improves online
latency relative to Jacobi under this contract, while direct solution is
faster on this 29-degree-of-freedom system. Model startup (0.410 s), offline
training, and data generation are excluded, so these are warm-model online
costs, not a lifetime-cost or large-scale comparison.

\FloatBarrier
\section{A scaled audit of the official NOWS Darcy implementation}
\label{app:nows-scaled}

\paragraph{Scope and implementation fidelity.}
We use the official NOWS Darcy implementation \citep{nows2026}, retaining its
four-block FNO architecture, width 32, 32 Fourier modes, coefficient preprocessing,
hard homogeneous Dirichlet transform, VINO variational energy, and finite-difference
matrix construction. We call the upstream training function with its data term
disabled. Source reference solutions are still generated for diagnostic logging,
so their cost is not assumed to be zero. This is a scaled audit of official code,
not a reproduction of the published NOWS performance tables.

\paragraph{Frozen training and test protocol.}
The protocol fixes 256 source training examples, 32 source validation examples,
three optimization seeds sharing that source dataset, and 100 epochs per seed.
These settings reduce the upstream default training budget; reaching 100 epochs
does not establish training saturation. The final epoch is fixed in advance,
without target-test checkpoint selection. The test design contains 48 physical
tasks: 16 in distribution, 16 with changed coefficient contrast, and 16 with
changed roughness. Each is evaluated on grids with 64 and 128 points per axis,
at true-relative-residual tolerances $10^{-6}$ and $10^{-8}$, with a cap of
2400 numerical iterations. The larger grid resamples the same base coefficient
field by nearest-neighbor interpolation; it is not an independently sampled
continuous-field discretization-invariance experiment.

\paragraph{Controls and acceptance.}
Each learned initial state is corrected by CG, Poisson-PCG, and AMG-PCG.
Zero and coarse-grid initializations are paired controls for each corrector;
sparse direct solution is a further complete numerical control. The reference
solution is computed outside measured deployment and never enters prediction
or stopping. Acceptance uses the verified residual of the same float64 system.
Failures and capped cases remain in the stated denominators. Residual acceptance
and relative error against the discrete direct solution are reported separately.

\paragraph{Counting and statistical unit.}
The design has 9216 timing records, corresponding to 48 tasks, two grids,
two tolerances, 16 method/seed configurations, and three timed repetitions.
These records are not 9216 independent problems. We first take the timing median
within each task/seed configuration, then average the three neural optimization
seeds within each physical task. Numerical controls are not triplicated over
neural seeds. We report each stratum, grid, tolerance, and hardware node
 separately for timing. Accuracy and iteration counts may pool the disjoint
physical tasks across hosts, retaining the full 16-task stratum denominator.
For timing, the node-assigned task count is reported separately from that global
denominator: a task completed on another host is not a missing task on this host.
Across-task medians and empirical 95th percentiles summarize latency; the latter
are descriptive summaries of the stated node subgroup, not tail-risk
guarantees. Paired task bootstrap intervals condition on the fitted models and
source dataset and are not adjusted for multiple comparisons.

\paragraph{Measured cost contract.}
Every timed query includes matrix assembly, necessary input conversion and
transfer, initial-state construction, preconditioner setup, correction, and
true-residual verification. The model is already loaded and queries use batch
size one. Training, source-data generation, and cold process/model startup are
excluded from this online number and reported separately when measured.
Concurrent activity on other GPUs shares host resources; the actual execution
receipt must identify whether timings were collected in that environment.
We do not equate iteration savings with wall-clock savings or mix successful
subsets into a full-task speed ranking.
We report assembly, initialization, preconditioner setup, and solve/verification
components. The official matrix-construction routine is retained as part of
this contract, so its observed overhead cannot establish the cost of an optimized
Darcy implementation. Subtracting assembly from an existing timer does not
constitute a measured cached-matrix deployment. Component medians need not sum
to the median complete-query time.

\paragraph{Recovery validation.}
The runner uses atomic checkpoints containing model, optimizer, scheduler,
CPU/CUDA/NumPy random states, source-data hashes, and completed epochs.
The scientific protocol is unchanged by a restart. An engineering-only
preflight compares three uninterrupted epochs to one epoch followed by a
serialized restart and two further epochs using two seeds of a tiny test configuration.
This gate passed on the execution hardware for both engineering seeds, with bitwise-equal model parameters and equal optimizer, scheduler, and CPU random states.
Evaluation restarts preserve SHA256-verified completed task/grid units.
Scheduling-only sharding retains both grids of each task, with shard assignment
$({\tt stratum\_index}\times16+{\tt task})\bmod {\tt shard\_count}$.
After removing the injected filter and lock wrapper, the scientific evaluation
function's abstract syntax tree must match the frozen runner exactly. The task
IDs, per-unit random method order, warmup, and measured repetitions are unchanged.
Exclusive unit locks and separate shard receipts prevent duplicated units.
The global acceptance manifest remains 96 task/grid files and 9216 timing rows.
Engineering interruptions and repeated resource allocations are reported
separately from per-query latency, without counting a resumed task twice.

\paragraph{Completion and numerical outcomes.}
All 48 physical tasks, 96 task/grid files, and 9,216 timed records passed the
exact manifest and file-hash audit; every recorded solve satisfied its own
true-relative-residual tolerance. The 3,072 seed-specific endpoints collapse
to 1,920 task-level method endpoints, not additional physical tasks. There were
no missing units, duplicate configurations, or numerical changes across
repetitions. All three models completed the fixed 100 epochs.
At $10^{-8}$, the complete 18-row comparison in
Table~\ref{tab:nows-iterations} shows correction-dependent gains. For CG on
128-point grids, learned initialization saves 20.15 updates versus zero in
distribution and 15.46 under roughness shift, but adds 38.71 under contrast
shift. Coarse initialization requires fewer updates in all nine
stratum/corrector combinations at this grid and tolerance; all corresponding
paired intervals lie below zero for the ``coarse minus learned'' difference.
Table~\ref{tab:nows-accuracy} separately reports prediction and final discrete
errors. The maximum learned-path final relative error at $10^{-8}$ is
$7.44\times10^{-10}$; this observation is not a continuous-PDE error certificate.

\paragraph{Timing and stronger controls.}
Timing is partitioned over the three observed hosts. For 128-point ID tasks,
Host-A contributes six and Host-B ten tasks; Host-B contributes all 16 contrast and all
16 roughness tasks. The 64-point ID stratum additionally has one task on Host-C.
A singleton subgroup cannot support a population uncertainty claim: its
bootstrap interval is degenerate. Tables and complete CSV files retain this
subgroup without inflating its sample size.
At the fixed display of grid128 and $10^{-8}$, none of the 12
host/stratum/corrector subgroups has a paired time-saving interval wholly above
zero versus either zero initialization or direct solution. This is neither an
equivalence test nor a statement that every individual task is slower.
The learned path does, however, save approximately 0.300--0.330\,s relative to
this implementation's coarse initialization in all 12 subgroups, with positive
paired intervals. Coarse initialization constructs and solves an additional
half-resolution system inside its timer. Its lower iteration count therefore
does not imply a lower complete-query cost.
At the looser $10^{-6}$ tolerance, grid128 Poisson-PCG has small positive
intervals against direct solution in three multi-task host subgroups, but no
positive interval against its matched zero-initialized corrector; those
contrasts do not establish an advantage attributable to learned initialization.

\paragraph{Assembly and offline accounting.}
At grid128, matrix-assembly component medians are approximately 1.27--1.30\,s
and dominate the retained complete-query contract. For the ten ID tasks on
Host-B at $10^{-8}$, direct, zero-CG, and learned-CG total medians are respectively
1.3322, 1.3591, and 1.3661\,s. The learned-CG initialization and solve/verification
component medians are 0.00427 and 0.06448\,s. The component breakdown in
Table~\ref{tab:nows-cost} prevents treating this as a benchmark of the best
available matrix-assembly implementation.
The three 100-epoch training-function timers are 12.859, 12.721, and
12.723\,s. They include validation inside the training function and device
synchronization, but exclude model/optimizer construction, checkpoint I/O,
final post-loop validation, preflight, and scheduling. They are not complete
offline costs. Shared source/reference generation and compressed cache writing
take a separately measured 94.823\,s. Engineering interruptions and allocation
costs are reported in the campaign receipt, not absorbed into these timers.
The final scheduling-only run retained 13 completed units unchanged and added
83, without changing the 96-unit scientific manifest. Concurrent GPU work
shared host resources; no whole-node isolation is claimed.

\paragraph{Interpretation and full reporting.}
This native-objective audit closes an objective-mismatch gap in the original
residual-trained study. Under the reduced training budget tested here, iteration and cost rankings vary with the distribution shift, corrector, and initialization overhead.
HINTS remains a separate positive example of repeated learned corrections;
its equation, resolution, hardware, and mechanism are not pooled with this
warm-start audit. All configurations, both tolerances, per-seed results,
acceptance denominators, and per-host component timings are retained in the
accompanying CSV and workbook. The fixed appendix display includes every
stratum and both grids for numerical comparisons at $10^{-8}$, and every
observed host/stratum subgroup for the five stated timing paths at grid128.

\begin{table}[t]
\centering\small
\caption{Scaled official NOWS audit at $10^{-8}$. Neural counts require every optimization seed to pass. Positive $\Delta$ means that the numerical control takes more updates. Each paired cell gives the mean, task-bootstrap interval, and number of jointly accepted tasks; any incomplete acceptance remains a conditional comparison. All strata and both grids are shown.}
\label{tab:nows-iterations}
\resizebox{\textwidth}{!}{\begin{tabular}{lllrlll}
\toprule
Shift & Grid & Corrector & Neural accepted & Neural updates & $\Delta$ vs zero [95\% CI]; $N$ & $\Delta$ vs coarse [95\% CI]; $N$\\
\midrule
ID & 64 & CG & 16/16 & 296.38 & 12.31 [10.75, 13.79]; 16 & -2.88 [-4.85, -0.98]; 16 \\
ID & 64 & Poisson-PCG & 16/16 & 14.02 & 1.92 [1.81, 2.00]; 16 & 0.10 [-0.02, 0.27]; 16 \\
ID & 64 & AMG-PCG & 16/16 & 7.10 & 1.02 [0.94, 1.12]; 16 & -0.04 [-0.10, 0.00]; 16 \\
ID & 128 & CG & 16/16 & 613.23 & 20.15 [17.12, 23.23]; 16 & -38.60 [-43.23, -34.17]; 16 \\
ID & 128 & Poisson-PCG & 16/16 & 14.94 & 1.50 [1.25, 1.75]; 16 & -0.38 [-0.62, -0.12]; 16 \\
ID & 128 & AMG-PCG & 16/16 & 8.62 & 1.12 [0.94, 1.31]; 16 & -0.62 [-0.81, -0.44]; 16 \\
Contrast & 64 & CG & 16/16 & 908.52 & -18.83 [-22.85, -14.96]; 16 & -39.40 [-47.73, -30.79]; 16 \\
Contrast & 64 & Poisson-PCG & 16/16 & 38.33 & 0.29 [0.10, 0.48]; 16 & -1.02 [-1.31, -0.73]; 16 \\
Contrast & 64 & AMG-PCG & 16/16 & 10.10 & -0.04 [-0.10, 0.00]; 16 & -0.54 [-0.77, -0.31]; 16 \\
Contrast & 128 & CG & 16/16 & 1886.02 & -38.71 [-46.54, -31.71]; 16 & -161.83 [-180.23, -145.14]; 16 \\
Contrast & 128 & Poisson-PCG & 16/16 & 41.62 & -0.12 [-0.38, 0.15]; 16 & -1.88 [-2.19, -1.54]; 16 \\
Contrast & 128 & AMG-PCG & 16/16 & 11.77 & -0.15 [-0.35, 0.04]; 16 & -1.08 [-1.19, -1.00]; 16 \\
Roughness & 64 & CG & 16/16 & 296.04 & 8.33 [6.46, 10.15]; 16 & -2.23 [-5.25, 0.56]; 16 \\
Roughness & 64 & Poisson-PCG & 16/16 & 15.00 & 2.00 [2.00, 2.00]; 16 & 0.06 [0.00, 0.19]; 16 \\
Roughness & 64 & AMG-PCG & 16/16 & 7.90 & 0.92 [0.73, 1.08]; 16 & 0.10 [0.00, 0.27]; 16 \\
Roughness & 128 & CG & 16/16 & 610.04 & 15.46 [9.92, 20.23]; 16 & -27.35 [-33.10, -21.60]; 16 \\
Roughness & 128 & Poisson-PCG & 16/16 & 15.98 & 1.02 [1.00, 1.06]; 16 & -0.67 [-0.88, -0.44]; 16 \\
Roughness & 128 & AMG-PCG & 16/16 & 8.98 & 0.90 [0.73, 1.00]; 16 & -0.23 [-0.44, -0.04]; 16 \\
\bottomrule\end{tabular}}
\end{table}

\begin{table}[t]\centering\small
\caption{Initial prediction error and final discrete verification for the scaled NOWS audit at $10^{-8}$. Initial error is shown from the fixed CG configuration. Final error/residual are maxima across the three correctors, tasks and optimization seeds. Residual acceptance is not an error certificate for the continuous PDE.}
\label{tab:nows-accuracy}
\begin{tabular}{llrrrr}\toprule
Shift & Grid & Mean initial error & Max final error & Max true residual & Min accepted\\\midrule
ID & 64 & 2.14e-02 & 6.47e-10 & 9.99e-09 & 16/16 \\
ID & 128 & 2.26e-02 & 4.18e-10 & 1.00e-08 & 16/16 \\
Contrast & 64 & 5.71e-01 & 7.44e-10 & 9.99e-09 & 16/16 \\
Contrast & 128 & 5.76e-01 & 4.39e-10 & 9.99e-09 & 16/16 \\
Roughness & 64 & 2.56e-02 & 4.88e-10 & 9.99e-09 & 16/16 \\
Roughness & 128 & 3.47e-02 & 4.82e-10 & 9.99e-09 & 16/16 \\
\bottomrule\end{tabular}\end{table}

\begin{table}[t]\centering\scriptsize
\caption{Within-host online cost in milliseconds for the fixed grid128, $10^{-8}$ display. $N$ is the verified task subset assigned to that host, not the global 16-task stratum denominator. Successful-subset numbers are conditional when acceptance is incomplete. The fixed display methods are direct, zero/VINO with CG, and zero/VINO with AMG-PCG. Component medians need not sum to total median.}
\label{tab:nows-cost}
\resizebox{\textwidth}{!}{\begin{tabular}{lllr rrrrrr}\toprule
Host & Shift & Method & Accepted/$N$ & Total med. & Total P95 & Assembly & Init. & Setup & Solve/verify\\\midrule
Host-B & ID & Direct & 10/10 & 1332.238 & 1367.106 & 1292.957 & 0.087 & 0.005 & 39.159 \\
Host-B & ID & Zero+CG & 10/10 & 1359.059 & 1372.133 & 1291.839 & 0.086 & 0.002 & 66.365 \\
Host-B & ID & VINO+CG & 10/10 & 1366.133 & 1381.798 & 1300.307 & 4.266 & 0.013 & 64.482 \\
Host-B & ID & Zero+AMG-PCG & 10/10 & 1336.855 & 1363.314 & 1293.188 & 0.089 & 24.267 & 19.904 \\
Host-B & ID & VINO+AMG-PCG & 10/10 & 1345.768 & 1360.770 & 1300.535 & 4.268 & 23.989 & 17.076 \\
Host-B & Contrast & Direct & 16/16 & 1329.791 & 1340.329 & 1290.252 & 0.086 & 0.002 & 39.278 \\
Host-B & Contrast & Zero+CG & 16/16 & 1477.182 & 1525.945 & 1287.808 & 0.086 & 0.003 & 183.396 \\
Host-B & Contrast & VINO+CG & 16/16 & 1488.822 & 1531.968 & 1296.128 & 4.254 & 0.007 & 187.881 \\
Host-B & Contrast & Zero+AMG-PCG & 16/16 & 1335.259 & 1387.074 & 1288.813 & 0.086 & 23.346 & 21.903 \\
Host-B & Contrast & VINO+AMG-PCG & 16/16 & 1342.350 & 1358.913 & 1290.746 & 4.255 & 23.168 & 21.951 \\
Host-B & Roughness & Direct & 16/16 & 1321.574 & 1354.608 & 1281.185 & 0.085 & 0.002 & 38.963 \\
Host-B & Roughness & Zero+CG & 16/16 & 1346.721 & 1380.966 & 1280.268 & 0.084 & 0.003 & 64.726 \\
Host-B & Roughness & VINO+CG & 16/16 & 1344.341 & 1368.938 & 1277.486 & 4.222 & 0.003 & 63.019 \\
Host-B & Roughness & Zero+AMG-PCG & 16/16 & 1319.319 & 1348.227 & 1274.797 & 0.086 & 23.915 & 19.875 \\
Host-B & Roughness & VINO+AMG-PCG & 16/16 & 1330.617 & 1349.704 & 1285.243 & 4.214 & 24.033 & 18.007 \\
Host-A & ID & Direct & 6/6 & 1327.124 & 1333.577 & 1292.487 & 0.095 & 0.007 & 34.410 \\
Host-A & ID & Zero+CG & 6/6 & 1358.266 & 1367.124 & 1292.053 & 0.093 & 0.017 & 66.077 \\
Host-A & ID & VINO+CG & 6/6 & 1364.059 & 1370.590 & 1296.182 & 4.124 & 0.003 & 64.416 \\
Host-A & ID & Zero+AMG-PCG & 6/6 & 1333.307 & 1341.165 & 1290.352 & 0.095 & 25.031 & 19.951 \\
Host-A & ID & VINO+AMG-PCG & 6/6 & 1331.903 & 1346.914 & 1285.992 & 4.129 & 25.048 & 17.783 \\
\bottomrule\end{tabular}}
\end{table}

\FloatBarrier
\section{Official-state boundary interface audit}
\label{app:boundary-interface}
\paragraph{Nonzero Dirichlet conditions through a verified numerical interface.}
We freeze the first 32 official HINTS Poisson test tasks and pair the boundary
values $(0,0)$, $(1,0)$, $(0,1)$, and $(1,-1)$ on the same coefficient, forcing,
and 30-element grid. These are 128 paired configurations of 32 previously
evaluated tasks, not 128 independent tasks or an additional independent sample.
The official discretization is the conservative operator $-(ku')'$, with
arithmetic face coefficients. An independently assembled full matrix, including
the boundary rows, has an interior block exactly equal to the official one.
For $g(x)=a(1-x)+bx$, we solve the homogeneous-boundary correction equation for
$v=u-g$ with right-hand side $f_{\mathrm{int}}-B_{\mathrm{int}}g$.
The latter correction is essential for variable coefficients: affine $g$ does
not generally have a zero conservative operator. Full-system and lifted direct
references agree to a maximum relative field difference of $2.731\times10^{-14}$
over all 128 configurations.

We retain the official pretrained weights, relaxation $2/3$, and one neural
correction every 25 total updates. Both iterative paths start from $v=0$.
Acceptance requires the true interior residual divided by the original forcing
norm to be at most $10^{-8}$ and exact endpoint values, with a cap of 10,000
updates. The original-forcing denominator is shared across the four boundary
conditions of a task; residuals with eliminated and lifted right-hand-side
denominators are additionally archived. All 384 method endpoints pass, with
zero boundary error. The zero-boundary HINTS/Jacobi update counts and acceptance
exactly reproduce the prior 64 frozen endpoints.
For the four boundary conditions in the order above, HINTS requires
137.438, 140.531, 143.406, and 146.094 mean total updates, versus
4966.281, 4899.750, 4983.969, and 5032.875 for Jacobi.
HINTS totals decompose into 132.250, 135.125, 137.938, and 140.500 numerical
updates plus 5.188, 5.406, 5.469, and 5.594 neural calls.
Every direct solve also passes. The maximum HINTS relative discrete solution
errors are $1.043\times10^{-8}$, $7.019\times10^{-10}$,
$5.314\times10^{-10}$, and $9.528\times10^{-10}$, respectively;
nonzero boundary values change the solution norm, so these relative errors
alone should not be interpreted as improved accuracy.
This establishes a correctly translated numerical interface in the tested
range, not learned generalization to new boundary conditions.
No formal cost claim is made, and mixed/Neumann conditions, other geometries,
and untested boundary amplitudes remain outside this experiment.

\FloatBarrier
\section{Extended literature coverage and scope}
\label{app:literature-scope}\label{app:extended-literature}

\paragraph{Search boundary.}
This is a structured scoping audit, not a claim to an exhaustive systematic
review. We checked the primary-source metadata of all 63 references in the
previous manuscript and added 41 records selected for direct methodological
or evaluation relevance, with a search cutoff of 24 September 2026.
Searches combined PINN, neural operator, warm start, preconditioner,
Krylov, meta-learning, transfer, discretization, geometry, benchmark, and cost
terms. Backward searches followed references in the existing review and key
hybrid-solver papers; targeted forward searches used method names and newer
papers that discuss those methods. Publisher records, official proceedings,
author manuscripts, and official code repositories take precedence over
aggregators. The accompanying catalogue separates publication status,
metadata verification, and relevance assessment. This process does not
estimate recall over the field, and metadata checks do not imply an independent
full-text replication audit of every listed work.

\paragraph{PINN optimization and reliability.}
Reviews of physics-informed machine learning describe broad method families
\citep{karniadakis2021,toscano2025,torres2025,deryckreview2024}.
Optimization pathology, operator preconditioning, precision, and training
practice are distinct contributors to observed PINN performance
\citep{ntk2020,gradient2021,expert2023,precondition2023,fp642025}.
Exact boundary imposition and domain decomposition change the approximation
space and optimization problem \citep{hardbc2021,fbpinn2021}.
Generalization analyses require PDE-dependent stability and sampling assumptions
\citep{mishra2020,mishrainverse2020}; Bayesian formulations address noisy-data
uncertainty through a different inference objective \citep{bpinn2020}.
These issues should not be collapsed into conditioning of the numerical
corrector after training. Moreover, observational accuracy and the evidence
needed for a scientific claim are different levels of assessment
\citep{wangevidence2026}. The present protocol addresses a specified numerical
deployment question, rather than establishing universal physical validity.

\paragraph{Operator representations and discretization.}
Operator-learning reviews place learned function-space maps within scientific
simulation and design \citep{aziz2024,operator2023}.
MeshGraphNets, Geo-FNO, GINO, Transolver, and CNO use different mechanisms for
mesh, geometry, or resolution handling
\citep{meshgraph2020,geofno2022,gino2023,transolver2024,cno2023}.
Neural-HSS is a recent structured Green-function approach, inspected here in
its preprint form \citep{neuralhss2026}.
Continuum approximation assumptions and finite-grid implementation errors
remain separate questions \citep{continuousdisc2024,mismatch2025}.
PINTO and geometry-aware physics-informed operators further condition on
initial/boundary data or geometry \citep{pinto2024,pigano2024}; their capabilities
cannot be attributed to an arbitrary fixed-grid CNN merely because both retain
weights. Similarly, analysis of what physics-informed DeepONets learn concerns
their training and representations, not a guarantee that a small surrogate
error implies efficient numerical correction \citep{williams2024}.

\paragraph{Hybrid numerical learning.}
Convergence-aware learned PDE iterations, learned discretizations, and
solver-feedback training place learning at different points inside numerical
computation \citep{neuralpde2019,mlcfd2021,solverloop2020}.
Learned algebraic multigrid concerns transfer operators; neural incomplete
factorization and related graph methods concern preconditioner construction
\citep{learnedamg2020,neuralif2024,trifonov2026,gnp2025}.
Learned incomplete factorization for GMRES extends the problem to nonsymmetric
systems \citep{neuralilu2025}; convolutional and sparse-approximate-inverse
preconditioners introduce further structural and hardware choices
\citep{sappl2026,spai2025}.
The recent DualNMG result is cited as an ICLR workshop paper, not as a main-track
publication \citep{dualnmg2026}.
These methods complement the warm-start literature rather than form a single
interchangeable baseline. NOWS is a published 2026 journal article whose
preprint appeared in 2025 \citep{nows2026}.
Its original training objective, architecture, solver configurations, and data
distribution must be distinguished from a reduced-budget adaptation.
Two further direct neighbors were recovered by cross-checking earlier project
plans: spectrally safe Newton warm starts and error-conditioned neural solvers
\citep{spectralsafe2026,ens2026}, both cited as 2026 preprints.
Their state-dependent Jacobian and learned recurrent-correction mechanisms
preclude claiming that prediction--solver mismatch or residual feedback is new.
Our fixed-SPD interventions instead isolate the initial error's spectral weights
while leaving the operator unchanged. A unique-solution nonlinear extension
does not reproduce a loss of Jacobian definiteness or nonlinear branch selection.
Backward citation checks additionally identify neural-initialized finite-element
Newton, learned Newton updates, fixed-point neural preconditioning, and
Cholesky-factor learning inside Newton--Kantorovich iterations
\citep{nin2026,nino2024,npnewton2025,machineprecision2025}.
These retain different objects and have different training and convergence
assumptions. DeepONet-based direct and trunk-basis preconditioning likewise
distinguish learned inverse actions from learned correction subspaces
\citep{deeponetpc2025}.

\paragraph{Adaptation, prompts, and foundation models.}
One-shot transfer, meta-learning, and learned optimization reuse different
objects and permit different target-time computations
\citep{oneshot2021,metapde2022,metaloss2021,dats2024,learnedsolver2024}.
Poseidon, DPOT, PDEformer, multiple-physics pretraining, PROSE, and Walrus target
different forms of breadth, conditioning, and transfer
\citep{poseidon2024,dpot2024,pdeformer2024,mpp2023,prose2025,walrus2025}.
Scaling studies provide evidence about transfer under their own pretraining
and fine-tuning protocols \citep{scaling2023}.
In-context operators and their PDE-generalization studies use problem examples
as information at inference \citep{icon2023,icongeneral2024}.
Equation recasting and operator splitting add structured composition or
target-time choices \citep{recast2026,splitting2026}.
Their classification requires recording prompts, observed trajectories,
searches, and auxiliary solvers in addition to parameter updates.
A spectral audit of in-context networks studies local response properties
\citep{spectral2026}; this is distinct from the matrix-weighted initial-error
measure underlying linear Krylov correction.

\paragraph{Generative and program-level reuse.}
Physics-informed diffusion, DiffusionPDE, and CoCoGen retain generative priors
and use different conditioning or consistency mechanisms
\citep{pidm2024,diffusionpde2024,cocogen2023}.
Distillation can change the inference budget without making the original
generative sampling cost disappear \citep{piddm2025}.
Language-guided PINN construction, evolutionary optimization, and agentic
algorithm discovery may retain executable procedures rather than only a
solution predictor \citep{langpinn2025,evolution2025,evopinn2026}.
PDEAgent-Bench evaluates solver generation across metrics and libraries
\citep{pdeagent2026}. Including these topics prevents an artificially narrow
definition of reuse, but the controlled experiments do not test every one of
these application settings or establish a ranking over them.

\paragraph{Cost and correctness boundaries.}
PDEBench, PINNacle, and The Well expand benchmark diversity
\citep{pdebench2022,pinnacle2023,well2024}; established baseline critiques require
equal accuracy and credible numerical competitors
\citep{weakbaseline2024,grossmann2024}.
Classical reusable factorizations, reduced bases, and Krylov subspaces belong
in many-query cost reasoning \citep{hesthaven2016,recycling2006}.
GPU-oriented neural preconditioners and TensorGalerkin make hardware placement
material to a wall-clock comparison \citep{spai2025,tensorgalerkin2026}.
For nonlinear problems, an initial guess can affect branch selection, not only
iteration count \citep{ancestor2026}. Accordingly, a residual tolerance certifies
the tested discrete equation under stated conditions; it does not by itself
certify mesh convergence, the intended nonlinear branch, or external physical
validity. 

\paragraph{Additional direct methodological precedent in this revision.} \citet{sambharya2024} train neural warm starts through fixed-point optimization, analyze generalization, and evaluate iteration and time reductions. This direct precedent is added to the original 104-entry catalogue; the present bibliography therefore contains 105 entries. The original 18 version-specific protocol records are not expanded by this metadata addition.

\FloatBarrier
\section{Frozen-library model selection on new tasks}
\label{app:selection-study}
\paragraph{Question and frozen design.}
This revision-stage experiment tests a decision suggested by the earlier audit: with the same candidate library and validation information, does selection by actual correction work improve held-out work relative to selection by prediction error? It adds no training and is not retrospectively described as part of an originally preregistered study. The protocol, candidate hashes, tie rules, failure handling, task identifiers, and quality thresholds were frozen before the new results. After all validation units completed, choices were written to a locked file before test data were generated.

Each of two diffusion source libraries contains 12 previously trained 6000-update candidates: CNN/FNO, data/residual-augmented objectives, and three optimization seeds. Earlier source-validation choices of residual weight are unchanged. There are 32 new validation and 64 new test tasks in each of ID and joint shift, with their continuous coefficient/forcing realizations shared across grids $n=31,63$. Thus there are 64 validation and 128 test physical tasks in total, not one new cohort per model, grid, source, or corrector. Validation IDs begin at 22100000/22200000 for ID/joint; test IDs begin at 22110000/22210000. All use the core generator. Zero and coarse are controls, not candidates in the learned-state library.

\paragraph{Selection, acceptance, and failures.}
Prediction selection minimizes mean validation relative Euclidean solution error. Work selection lexicographically minimizes validation failure count, mean penalized updates, and the fixed candidate label. A failed solve receives a score of 1201, one greater than the 1200-update cap; it is not recorded as a successful 1201-update solve. CG and Poisson-PCG require algorithmic success, finite values, and a recomputed physical relative residual at most $10^{-8}(1+10^{-6})$. Validation predictions and labels are shared; no test observation selects a model. Secondary residual selection minimizes the initial physical relative residual. All 19,968 validation/test endpoints pass the declared gate. The analysis independently reconstructs each locked choice from the raw validation records.

\paragraph{Costs and statistical unit.}
Both main policies receive the same 32 reference labels and predictions per stratum/grid. Work scoring additionally requires per-task preconditioner preparation and all 384 candidate corrections per library/corrector. Its measured incremental CPU cost is 0.19--4.97 seconds per configuration, excluding small sorting and file-I/O costs. This is an incremental validation measurement, not full offline lifecycle cost. Each timed test path freshly assembles the matrix, constructs or predicts the initial state, builds the preconditioner, corrects, and checks true residual; models remain resident. Task generation, model loading, reference solution, and final error diagnostics are outside that boundary. Path order is shuffled within task. There is one timing observation per task/path, so times are descriptive with no timing intervals or claim of time-optimal selection.

Reported work contrasts are paired over the same 64 test tasks, with 4000 percentile-bootstrap resamples (seed 20260925), separately for source, distribution, grid, and corrector. They condition on the frozen library and this validation draw; they do not include variation from regenerating source data, retraining models, or repeating validation selection. The 16 main intervals are unadjusted descriptive intervals. A zero-width interval when both policies select the same state is an identity on these executions, not proof of equivalence between distinct policies in general.

\paragraph{Accurate trained-state ranking reversals.}
Within each task and source library, compare all 66 candidate pairs. A strict ranking reversal occurs when a smaller initial summary accompanies more penalized updates; tied work is a non-reversal. Summary differences at most $10^{-12}$ are excluded from that summary's comparable-pair denominator. For joint dominance, all three summaries must be strictly smaller in the same direction. The predeclared diagnostic strata require both prediction errors to be at most 5\%, 10\%, or 25\%, or impose no quality restriction. These true-error strata use test references for diagnosis, never for model selection. We average reversal fractions within task before task bootstrapping; a task without an eligible pair is excluded only from that stratum's conditional fraction, with its count explicitly retained. Pair counts are not independent sample sizes.

Table~\ref{tab:selection-accurate-states} gives the ID/$n=63$, 5\% stratum. For CG, joint-dominance reversals number 72/362 and 101/408 eligible pairs, while the mean taskwise fractions are 23.50\% and 23.87\%; these are intentionally different summaries. Poisson-PCG has no observed joint-dominance reversal on the same subsets, with 245/276 work ties in the two libraries. Zero observed events and a degenerate bootstrap interval do not imply zero population risk. This is direct evidence of ranking mismatch in trained states, but it does not isolate a causal spectral mechanism. The full quality strata, secondary selector, controls, chosen labels, costs, and raw records are included in the executable supplement.

\begin{table}[t]
\centering\small
\caption{Real learned-state reversals on the 64 new ID tasks at $n=63$,
restricted to pairs for which both relative prediction errors are at most
5\%. The estimate is the mean of per-task reversal fractions with a task-bootstrap
95\% interval, not the pooled-pair fraction. All rows have 64 eligible tasks.
Work ties are non-reversals; zero observed reversals do not bound population risk.}
\label{tab:selection-accurate-states}
\begin{tabular}{cclrr}
\toprule
Source & Corrector & Smaller summary & Reversals/eligible pairs & Mean task fraction [95\% CI] \\
\midrule
0 & CG & Solution error & 423/935 & 45.47\% [41.51, 49.38] \\
0 & CG & Initial residual & 260/935 & 27.81\% [24.01, 32.03] \\
0 & CG & Energy error & 243/935 & 26.04\% [22.08, 30.31] \\
0 & CG & All three (strict) & 72/362 & 23.50\% [17.61, 30.39] \\
0 & PCG & Solution error & 145/935 & 15.42\% [11.15, 20.10] \\
0 & PCG & Initial residual & 7/935 & 0.73\% [0.21, 1.35] \\
0 & PCG & Energy error & 2/935 & 0.21\% [0.00, 0.52] \\
0 & PCG & All three (strict) & 0/362 & 0.00\% [0.00, 0.00] \\
1 & CG & Solution error & 419/906 & 45.57\% [40.68, 50.47] \\
1 & CG & Initial residual & 259/906 & 28.75\% [25.16, 32.40] \\
1 & CG & Energy error & 247/906 & 27.19\% [23.80, 30.68] \\
1 & CG & All three (strict) & 101/408 & 23.87\% [18.58, 29.33] \\
1 & PCG & Solution error & 116/906 & 12.92\% [9.11, 17.19] \\
1 & PCG & Initial residual & 8/906 & 0.94\% [0.36, 1.61] \\
1 & PCG & Energy error & 2/906 & 0.21\% [0.00, 0.52] \\
1 & PCG & All three (strict) & 0/408 & 0.00\% [0.00, 0.00] \\
\bottomrule
\end{tabular}
\end{table}

\begin{table}[t]
\centering\small
\caption{Absolute held-out updates for the prediction-selected and work-selected states. Each row is 64 accepted tasks; all paths use the same tolerance. Corresponding differences and intervals appear in Table~\ref{tab:selection-new-tasks}. Full controls and descriptive CPU timings accompany the raw-data analysis.}
\label{tab:selection-absolute-work}
\begin{tabular}{llclrr}\toprule
Tasks & Grid & Source & Corrector & Prediction-selected & Work-selected\\\midrule
ID & 31 & 0 & CG & 149.297 & 149.297\\
ID & 31 & 0 & Poisson-PCG & 17.109 & 16.984\\
ID & 31 & 1 & CG & 148.953 & 149.469\\
ID & 31 & 1 & Poisson-PCG & 17.047 & 17.031\\
ID & 63 & 0 & CG & 320.656 & 318.156\\
ID & 63 & 0 & Poisson-PCG & 18.359 & 18.031\\
ID & 63 & 1 & CG & 320.891 & 317.562\\
ID & 63 & 1 & Poisson-PCG & 18.391 & 18.016\\
Joint & 31 & 0 & CG & 271.438 & 271.438\\
Joint & 31 & 0 & Poisson-PCG & 52.484 & 52.484\\
Joint & 31 & 1 & CG & 271.781 & 271.938\\
Joint & 31 & 1 & Poisson-PCG & 52.469 & 52.562\\
Joint & 63 & 0 & CG & 594.984 & 594.984\\
Joint & 63 & 0 & Poisson-PCG & 55.766 & 55.766\\
Joint & 63 & 1 & CG & 596.922 & 596.375\\
Joint & 63 & 1 & Poisson-PCG & 55.812 & 55.703\\
\bottomrule\end{tabular}
\end{table}

\FloatBarrier

\section{Complete evidence map and comparison units}
\label{app:evidence-map}
The scope of the study is preserved across the main text, supplementary methods, and the aggregate workbook. An experiment family is not an independent replication of every other family: source states, task identities, budget prefixes, and grid representations are reused where declared. The units below must not be pooled into a single sample-size claim.

\begin{longtable}{p{0.27\linewidth}p{0.30\linewidth}p{0.34\linewidth}}
\toprule
Evidence family & Reported record coverage & Independence and interpretation \\
\midrule\endfirsthead
\toprule Evidence family & Reported record coverage & Independence and interpretation\\\midrule\endhead
Original supplied aggregates & Historical table retained in workbook & Original run-count semantics; not reconstructed into missing per-task weights or pooled with controlled waves.\\
Two core linear waves & 240 source trajectories; 144 adaptation trajectories; 494,592 correction endpoints & Main tasks: 256 plus 128. Budget checkpoints, seeds, grids, and tolerances are dependent.\\
Core timing and spectral records & 39,624 timing repetitions; 41,527 spectral records & Subsets and oracle constructions have their own units. The 41,463 expanded spectral records exclude nine follow-ups after one failed matching gate.\\
Completed deployment suites & 360 units; 16,008 JSON records, including 240 fresh-process records & Query expansion yields 25,584 final-query records, with 142 failures; none of these totals is an independent-task count.\\
Nested source count & 24 new trajectories plus 12 reused states; 3,840 accepted endpoints & One nested source pool; 64 new physical test tasks. Fixed update budget changes exposure per example.\\
Fixed target-information allocation & 108 actual trajectories; 648 fixed-rate checkpoint/policy states; 324 rate-selected states & 128 test tasks per equation, disjoint from selection. Three allocations reuse one 80-label pool per equation.\\
Initial nonlinear study & 6 source states; 9,072 endpoints; 27,216 timing repetitions & 112 tasks over seven strata; grids and seeds do not add tasks.\\
Longer nonlinear recipe & 6 newly trained states; 9,072 endpoints; 27,216 repetitions & Same 112 tasks and source pool, but a stretched learning-rate schedule.\\
Nonlinear confirmation & 4,536 endpoints; 13,608 repetitions & 56 new tasks reuse the longer-trained states; no additional training.\\
Official HINTS mechanism & 750 official cases; 6,750 test endpoints & The original 32 cases are a prefix, not extra samples; native 29-interior-DOF configuration.\\
HINTS warm timing & 480 accepted timed runs & 32 fixed tasks, five paths, three repeats; separate from oracle-diagnostic timing.\\
HINTS boundary lifting & 384 accepted endpoints; 128 bridge configurations & Same 32 tasks, four boundary conditions. No new independent tasks or formal timing claim.\\
Independent batch confirmation & 8 groups; 168 batch records; 5,376 accepted query receipts & 256 new tasks, grouped for timing inference; one fixed source state and node, excluding discovery groups.\\
Fresh-task model selection & 19,968 accepted validation/test endpoints & 64 validation and 128 new test physical tasks; two frozen 12-model libraries; paired grids and correctors.\\
Scaled official NOWS audit & 3 models, each 100 epochs; 96 task/grid units; 9,216 timed records & 48 tasks over three strata and two grids; one source pool. Timing is partitioned by host.\\
\bottomrule
\end{longtable}

The source workbook contains 11 worksheets: overview, original core summaries, complete expanded configurations, matched cross-grid configurations, deployment summaries, paired deployment contrasts, target information, source count, nonlinear results, native methods, and coverage/provenance. Their aggregate blocks accompany the manuscript source; the core and batch raw records and the new model-selection records additionally accompany the executable supplement. Individual rows may be aggregate configurations, repeated-control displays, or provenance entries. They are not automatically additional physical tasks. The exact file inventory, row counts, and input hashes are supplied separately.

\paragraph{Experimental scope.}
The controlled experiments cover static forward problems and supervised target adaptation. The literature map also includes irregular geometries, unstructured meshes, cross-equation transfer, nonmonotone branch selection, long-time evolution, inverse-posterior calibration, and executable-program discovery. Extending the empirical analysis to these settings requires corresponding task definitions and evaluation protocols. A further bridge to coordinate PINNs would compare source-initialized and scratch models under matched target-residual optimization.

\paragraph{Further directions.}
The results motivate work on predicting finite-tolerance correction cost from inexpensive features, accounting for paid probes in selection policies, allocating a fixed label budget between adaptation and validation, developing local-Jacobian estimates for nonlinear error and cost, and quantifying uncertainty under shared source states and tasks.

\section{Executable evidence and verification scope}
\label{app:artifact-scope}
\paragraph{Anonymous artifacts.}
The manuscript source contains all text, figures, tables, 105 bibliographic entries, exact-arithmetic checks, and 74 historical aggregate blocks. The separate executable supplement contains a core reconstruction directory, \texttt{repro/}, and the frozen-library experiment, \texttt{selection\_study/}. Their readmes specify environment versions, commands, task generation, checkpoint locations, timing boundaries, and reconstruction scope. Administrative author paths and machine names are removed or consistently aliased; measured values, task IDs, grouping, and execution order are preserved. The manifest hashes the distributed files.

\paragraph{Reconstructing the main table and batch confirmation.}
The core artifact includes 768 original task-level endpoints: 576 learned-state and 192 matched zero/coarse endpoints across both PDE families, two source libraries, three optimizer seeds, and three correctors per family. Analysis averages seeds within task, then reconstructs the paired-task summaries and 4000-resample intervals used for Table~\ref{tab:core}. It also recomputes initial error/residual/energy summaries and absolute work. The batch artifact preserves all eight new groups, 168 batch records, and 5,376 query receipts, including their run order. It reconstructs three-repeat group medians and paired group-bootstrap comparisons without pooling the two discovery groups. Regression checks compare reconstructed outputs with the original summaries; they are exact-record statistical reconstruction, not an independent new sample.

\paragraph{Executing models and solvers.}
The supplement includes the 48 recovered source-selected checkpoints needed by the original deployment loader, model definitions, deterministic task generators, correction and batch entry points, and source-selection/configuration records. A CPU reevaluation of all 768 core endpoints with the recovered weights passes acceptance and retains the diffusion CG-versus-preconditioned-CG sign pattern. Different PyTorch versions, devices, and numerical libraries do not produce bitwise identical predictions or iteration counts; the supplied validation receipt records these differences. Paper timings remain the original A100/CPU deployment observations. Neither a CPU rerun nor raw-record bootstrap replay reproduces those wall-clock measurements on another machine. The runnable batch entry preserves grouping and order and reports the new machine's costs separately.

\paragraph{New-task selection is fully separated.}
The new study includes its frozen protocol, checkpoint hashes, locked validation choices, all 19,968 validation/test endpoints, environment, deterministic task identifiers, complete analysis, and portable execution entry. Reanalysis checks the full expected Cartesian product, acceptance, and the selected labels. An additional first-acceptance check explicitly recomputes the true physical residual at every CG iterate for all 448 distinct selected $n=63$ test paths: all recorded acceptance counts agree with the first true-residual crossing. This is numerical validation of existing executions, not 448 additional independent tasks. Trained-state ranking denominators and dominance direction are separately checked.

\paragraph{Reproducibility scope.}
The supplement supports reconstruction of the supplied measurements and reevaluation of the included checkpoints. Historical training arrays, optimizer trajectories, and native-method checkpoints outside this package remain documented through protocols and aggregate results. The exact-arithmetic script verifies the stated CG constructions, while numerical acceptance checks apply to the specified discrete systems. Retraining and independent replication require the corresponding training data, execution environments, and experimental protocols.

\end{document}